\documentclass[acmsmall,screen,noacm,authorversion]{acmart} 

\AtBeginDocument{%
  }

\usepackage{amsmath}
\usepackage{amsfonts}
\usepackage{bm}

\def\eqref#1{equation~\ref{#1}}

\def\1{\bm{1}}

\def\vf{{\bm{f}}}

\def\vm{{\bm{m}}}

\def\vt{{\bm{t}}}

\def\vw{{\bm{w}}}

\def\vz{{\bm{z}}}

\DeclareMathAlphabet{\mathsfit}{\encodingdefault}{\sfdefault}{m}{sl}
\SetMathAlphabet{\mathsfit}{bold}{\encodingdefault}{\sfdefault}{bx}{n}

\newcommand{\softmax}{\ensuremath{\mathrm{softmax}}}

\usepackage[utf8]{inputenc} %
\usepackage{lipsum} %

\usepackage[T1]{fontenc}

\usepackage{etoolbox}
\newbool{includeappendix}
\setbool{includeappendix}{false}

\ifdefined\isoverfull
\else
\fi

\usepackage[]{color} %
\usepackage{xcolor} %

\definecolor{my-full-blue}{HTML}{1F77B4}

\definecolor{my-full-orange}{HTML}{FF7F0E}

\definecolor{my-full-green}{HTML}{2CA02C}

\definecolor{my-full-red}{HTML}{d62728}

\definecolor{my-full-purple}{HTML}{9467bd}

\colorlet{my-blue}{my-full-blue!30}
\colorlet{my-orange}{my-full-orange!30}
\colorlet{my-green}{my-full-green!30}
\colorlet{my-red}{my-full-red!30}
\colorlet{my-purple}{my-full-purple!30}

\usepackage{listings}

\usepackage{textcomp}

\usepackage{xcolor}

\usepackage[scaled=0.8]{beramono}

\definecolor{ckeyword}{HTML}{7F0055}
\definecolor{ccomment}{HTML}{3F7F5F}
\definecolor{cstring}{HTML}{2A0099}

\lstdefinestyle{numbers}{
	numbers=left,
	framexleftmargin=20pt,
	numberstyle=\tiny,
	firstnumber=auto,
	numbersep=1em,
	xleftmargin=2em
}

\lstdefinestyle{layout}{
	frame=none,
	captionpos=b,
}

\lstdefinestyle{comment-style}{
	morecomment=[l]//,
	morecomment=[s]{/*}{*/},
	commentstyle={\color{ccomment}\itshape},
}

\lstdefinestyle{string-style}{
	morestring=[b]",%
	morestring=[b]',%
	stringstyle={\color{cstring}},
	showstringspaces=false,%
}

\lstdefinestyle{keyword-style}{
	keywordstyle={\ttfamily\bfseries},
	morekeywords={
		function,
		constructor,
		int,
		bool,
		return,
		returns,
		uint
	},
	morekeywords = [2]{},
	keywordstyle = [2]{\text},
	sensitive=true,
}

\lstdefinestyle{input-encoding}{
	inputencoding=utf8,
	extendedchars=true,
	literate=
	{ℝ}{$\reals$}1%
	{→}{$\rightarrow$}1%
	{α}{$\alpha$}1%
	{β}{$\beta$}1%
	{λ}{$\lambda$}1%
	{θ}{$\theta$}1%
	{ϕ}{$\phi$}1%
}

\lstdefinestyle{escaping}{
	moredelim={**[is][\color{blue}]{\%}{\%}},
	escapechar=|,
	mathescape=true
}

\lstdefinestyle{default-style}{
	basicstyle=\fontencoding{T1}\ttfamily\footnotesize,
	style=numbers,
	style=layout,
	style=comment-style,
	style=string-style,
	style=keyword-style,
	style=input-encoding,
	style=escaping,
	tabsize=2,
	upquote=true
}

\lstdefinelanguage{BASIC}{
	language=C++,
	style=default-style
}[keywords,comments,strings]%

\definecolor{sh_comment}{rgb}{0.12, 0.38, 0.18 } 
\definecolor{sh_keyword}{rgb}{0.37, 0.08, 0.25}  
\definecolor{sh_string}{rgb}{0.06, 0.10, 0.98} 

\lstdefinelanguage{lmql}
{
  morekeywords={
    ARGMAX, BEAM, SAMPLE, FROM, WHERE, and, in, DISTRIBUTION, import
  },
  sensitive=false, %
  morecomment=[l]{//}, %
  morecomment=[s]{/*}{*/}, %
  morestring=[b]",
  commentstyle=\color{sh_comment}, %
  keywordstyle=[1]\color{sh_keyword}\textbf, %
  stringstyle=\color{sh_string}, %
}

\lstdefinelanguage{reactprompt}
{
  morekeywords={
    Act, Tho, Obs
  },
  sensitive=false, %
  morecomment=[l]{//}, %
  morecomment=[s]{/*}{*/}, %
  morestring=[b]",
  commentstyle=\color{sh_comment}, %
  keywordstyle=[1]\color{sh_keyword}\textbf, %
  stringstyle=\color{sh_string}, %
}

\usepackage[capitalize]{cleveref}

\crefformat{section}{\S#2#1#3}

\crefrangeformat{section}{\S#3#1#4\crefrangeconjunction\S#5#2#6}

\crefmultiformat{section}{\S#2#1#3}{\crefpairconjunction\S#2#1#3}{\crefmiddleconjunction\S#2#1#3}{\creflastconjunction\S#2#1#3}

\newcommand{\crefrangeconjunction}{--}

\crefname{listing}{Lst.}{listings}
\crefname{line}{Lin.}{Lin.}
\crefname{appendix}{App.}{App.}

\newcommand{\app}[1]{%
	\ifbool{includeappendix}{\cref{#1}}{the appendix}%
}
\newcommand{\App}[1]{%
	\ifbool{includeappendix}{\cref{#1}}{The appendix}%
}

\usepackage{xspace}
\usepackage{lipsum}
\usepackage{wrapfig}
\usepackage{enumitem}
\usepackage{caption}
\usepackage{subcaption}
\usepackage{algorithm2e}
\RestyleAlgo{ruled}
\usepackage{stmaryrd}
\usepackage{ifthen}
\usepackage{xfrac}
\usepackage{multirow}
\usepackage{makecell}
\usepackage{adjustbox}
\usepackage{tikz}
\usepackage{pgfplots}
\pgfplotsset{compat=1.16}

\usepackage{tablefootnote}

\newbool{extended}
\setbool{extended}{true}

\newcommand{\emptystring}{\ensuremath{\epsilon}}

\newcommand{\tool}{\textsc{LMQL}\xspace}
\newcommand{\voc}{\ensuremath{\mathcal{V}}}
\newcommand{\ebnfph}[1]{$\langle$#1$\rangle$}

\newcommand{\cd}{\textsc{FollowMap}}

\newcommand{\bslong}{Scripted Beam Search}
\newcommand{\paradigm}{\textsc{LMP}\xspace}
\newcommand{\paradigmlong}{Language Model Programming\xspace}

\newcommand{\scope}{\ensuremath{\sigma}}

\newcommand{\inc}[1][]{\ensuremath{\text{\textsc{inc}}{}\ifthenelse{\equal{#1}{}}{}{\!\left({#1}\right)}}}
\newcommand{\dec}[1][]{\ensuremath{\text{\textsc{dec}}{}\ifthenelse{\equal{#1}{}}{}{\!\left({#1}\right)}}}
\newcommand{\var}[1][]{\ensuremath{\text{\textsc{var}}{}\ifthenelse{\equal{#1}{}}{}{\!\left({#1}\right)}}}
\newcommand{\fin}[1][]{\ensuremath{\text{\textsc{fin}}{}\ifthenelse{\equal{#1}{}}{}{\!\left({#1}\right)}}}
\newcommand{\ann}[1][]{\ensuremath{\text{\textsc{tag}}{}\ifthenelse{\equal{#1}{}}{}{\!\left({#1}\right)}}}

\newcommand{\eos}{\textsc{eos}\xspace}

\newcommand{\eval}[2][]{\ensuremath{\llbracket{#2}\rrbracket{}_{\scope\ifthenelse{\equal{#1}{}}{}{[#1]}}}}

\lstdefinelanguage{lmql}{
  keywords = {in, and, or, over, not, beam, sample, argmax, from, where, distribute, if, elif, else, for, break, import, distribution, True, False},
  morestring=[b]",
}
\crefname{algocf}{Alg.}{Algs.}
\Crefname{algocf}{Algorithm}{Algorithms}
\crefalias{AlgoLine}{line}
\crefalias{algocfline}{line}
\crefname{line}{line}{lines}
\Crefname{line}{line}{lines}

\setbool{includeappendix}{true}
\setbool{extended}{true}
\setcopyright{rightsretained}
\acmPrice{}
\acmDOI{10.1145/3591300}
\acmYear{2023}
\copyrightyear{2023}
\acmSubmissionID{pldi23main-p737-p}
\acmJournal{PACMPL}
\acmNumber{PLDI}
\acmMonth{6}
\acmVolume{7}
\acmNumber{PLDI}

\begin{document}
		 
\title{Prompting Is Programming: A Query Language for\\ Large Language Models}
\renewcommand{\shorttitle}{Prompting Is Programming: A Query Language for Large Language Models}

\author{Luca Beurer-Kellner}
\email{luca.beurer-kellner@inf.ethz.ch}
\orcid{0000-0001-7734-3106}

\author{Marc Fischer}
\email{marc.fischer@inf.ethz.ch}
\orcid{0000-0002-4157-1235}

\author{Martin Vechev}
\affiliation{
	\institution{ETH Zurich}
	\country{Switzerland}
}
\email{martin.vechev@inf.ethz.ch}
\orcid{0000-0002-0054-9568}

\begin{abstract}

Large language models have demonstrated outstanding performance on a wide range of tasks such as question answering and code generation.
On a high level, given an input, a language model can be used to automatically complete the sequence in a statistically-likely way. Based on this, users prompt these models with language instructions or examples, to implement a variety of downstream tasks. Advanced prompting methods can even imply interaction between the language model, a user, and external tools such as calculators. However, to obtain state-of-the-art performance or adapt language models for specific tasks, complex task- and model-specific programs have to be implemented, which may still require ad-hoc interaction.

Based on this, we present the novel idea of Language Model Programming (LMP). LMP generalizes language model prompting from pure text prompts to an intuitive combination of text prompting and scripting. Additionally, LMP allows constraints to be specified over the language model output. This enables easy adaption to many tasks while abstracting language model internals and providing high-level semantics.

To enable LMP, we implement \tool (short for Language Model Query Language), which leverages the constraints and control flow from an LMP prompt to generate an efficient inference procedure that minimizes the number of expensive calls to the underlying language model.

We show that \tool can capture a wide range of state-of-the-art prompting methods in an intuitive way, especially facilitating interactive flows that are challenging to implement with existing high-level APIs. Our evaluation shows that we retain or increase the accuracy on several downstream tasks, while also significantly reducing the required amount of computation or cost in the case of pay-to-use APIs (26-85\% cost savings).

\end{abstract}

\begin{CCSXML}
<ccs2012>
<concept>
<concept_id>10011007.10011006.10011050</concept_id>
<concept_desc>Software and its engineering~Context specific languages</concept_desc>
<concept_significance>500</concept_significance>
</concept>
<concept>
<concept_id>10010147.10010178.10010179</concept_id>
<concept_desc>Computing methodologies~Natural language processing</concept_desc>
<concept_significance>500</concept_significance>
</concept>
<concept>
<concept_id>10010147.10010257</concept_id>
<concept_desc>Computing methodologies~Machine learning</concept_desc>
<concept_significance>500</concept_significance>
</concept>
</ccs2012>
\end{CCSXML}

\ccsdesc[500]{Software and its engineering~Context specific languages}
\ccsdesc[500]{Computing methodologies~Natural language processing}
\ccsdesc[500]{Computing methodologies~Machine learning}

\keywords{language model programming, prompt programming}

\maketitle

\ifbool{extended}{
\fancypagestyle{firstpagestyle}{%
\fancyfoot[R]{}
}
\fancyfoot[RO,LE]{}
}{}

\section{Introduction} \label{sec:intro}

Large Language Models (Large LMs - LLMs) \citep{VaswaniSPUJGKP17,DevlinCLT19,radford2019language,BrownMRSKDNSSAA20} have proven successful at various language-based tasks such as machine translation, text summarization, question answering, reasoning, code generation from text and many more.
Due to these results, LMs have become popular beyond the machine learning community and are slowly being integrated into many applications.

\paragraph{(Large) Language Models}
Internally, language models operate on tokens, which are different from how humans perceive language.
Given the tokenized version of some input, called the \emph{prompt}, a large language model predicts the next token.
That is, over a large vocabulary of tokens it assigns each a score or probability.
A \emph{decoding} procedure is then used, which by invoking the LM multiple times, computes a completion of the prompt.
Commonly, the goal is to determine (or approximate) the highest probability continuation, however, as producing a particular token might lower the probability, before a subsequent token increases it, decoding sometimes requires expensive search or backtracking strategies.
Nonetheless, LM-based text completion has shown to be very powerful and can be leveraged for a wide range of downstream applications.

\paragraph{Key Challenges in Using Language Models}
While the newer generation of language models can be prompted with examples or instructions in a conceptually simple manner, making the best use of these models and keeping up as new models are released requires a deep understanding of their internals, as well as the use of vendor-specific libraries and implementations. For example, as LMs operate on tokens, it can be hard to constrain the decoding procedure to a set of legal words or phrases.
Further, many prompting techniques can require back-and-forth interaction between the LM and the user (e.g. chatbots like ChatGPT \cite{openaiChatGPTOptimizing}) or very task-specific interfaces (e.g. to perform arithmetic calculations with external control logic). To implement these prompts, a lot of manual work and interaction with a model's decoding procedure is required, which restricts the generality of the resulting implementations.
Lastly, as an LM only produces one (sub-word) token at a time, completing a sequence may require many calls. 
Also, decoding becomes increasingly expensive as the prefix, the prompt, and the so-far generated response grow. Because of these factors, and as language models are typically very large neural networks, practical inference demands high computational costs and significant latency. In the case of pay-to-use APIs, such as OpenAI's GPT models, this results in high usage costs per query answered.

\begin{figure}
\centering
\begin{subfigure}[t]{0.49\textwidth}
\begin{lstlisting}
beam(n=3)
  "A list of good dad jokes. A indicates the "
  "punchline \n"
  "Q: How does a penguin build its house? \n"
  "A: Igloos it together. END \n"
  "Q: Which knight invented King Arthur's Round"
  "Table? \n"
  "A: Sir Cumference. END \n"
  "Q: [JOKE] \n"
  "A: [PUNCHLINE] \n"
from "gpt2-medium"
where
  STOPS_AT(JOKE, "?") and STOPS_AT(PUNCHLINE, "END")
  and len(words(JOKE)) < 20
  and len(characters(PUNCHLINE)) > 10
\end{lstlisting}
\vspace{-0.5em}
\caption{\tool query to generate a joke.} \label{fig:query-example:joke}
\end{subfigure}
\hfill
\vline
\hfill
\begin{subfigure}[t]{0.43\textwidth}
\begin{lstlisting}[numbers=left, xleftmargin=0.5em]
argmax
  "A list of things not to forget when "
  "travelling:\n"
  things = []
  for i in range(2):
    "- [THING]\n"
    things.append(THING)
  "The most important of these is [ITEM]."
from "EleutherAI/gpt-j-6B"
where
   THING in ["passport",
             "phone",
             "keys", ...] // a longer list
   and len(words(THING)) <= 2
\end{lstlisting}
\vspace{0.5em}
\caption{\tool query utilizing a python list.} \label{fig:query-example:list}
\end{subfigure}
\vspace{-0.5em}
	\caption{Two LMQL programs that demonstrate core features like scripted prompting, eager output constraining and validation, and prompting with control flow.}
  \label{fig:query-example}
\vspace{-1.2em}
\end{figure}

\paragraph{This work: Language Model Programming via \tool}
In this work, we propose the idea of language model programming, extending on natural language prompting by additionally allowing lightweight scripting and constraining of outputs.
This facilitates a front-end/back-end separation for LM prompting, i.e. allows a user to specify complex interactions, control flow, and constraints without requiring knowledge of an LM's internals such as tokenization, implementation, and architecture. Further, the constructed programs remain agnostic concerning the underlying LM, greatly improving portability. Overall, Language Model Programming (LMP) retains the simple natural-language-driven interface to LMs but additionally enables precise constraining, scripting, and efficient decoding, which, as of now, is not possible with existing high-level APIs.

To enable LMP, we present a novel language and runtime called the Language Model Query Language (\tool). \tool is a high-level language with declarative SQL-like elements and an imperative syntax for scripting. 
The underlying runtime is compatible with existing LMs and can be supported easily, requiring only a simple change in the decoder logic. \tool can be used to express a wide variety of existing prompting methods \cite{ReynoldsM21, wei2022chain,cobbe2021training,yao2022react, scholak2021picard, shin2021constrained} using simple, concise, and vendor-agnostic code. Further, purpose-designed evaluation semantics with support for partial evaluation and lookahead, enable us to optimize query execution end-to-end: \tool{} leverages user constraints and scripted prompts to prune the search space of an LM by masking, resulting in an up to 80\% reduction of inference cost.
We showcase two examples of simple \tool{} programs in \cref{fig:query-example}.

\paragraph{Main Contributions} Our core contributions are:
\begin{itemize}
\item We introduce the novel paradigm of language model programming, formulating and addressing several challenges that arise with recent LM prompting techniques (\cref{sec:overview}).
\item \tool, an efficient, high-level query language for LMs with support for scripted prompting and output constraining.
 (\cref{sec:query,sec:language}).
\item A formal model of eager, partial evaluation semantics based on so-called \emph{final and follow} abstractions. Using these, we can automatically generate model-specific token masks for LM decoding, given just a set of high-level constraints (\cref{sec:constraint_decoding}).
\item A comprehensive evaluation of \tool that shows how to express a wide range of common and advanced prompting techniques as simple and concise \tool programs, which also execute more efficiently, as \tool{} reduces inference cost and latency by 26-80\% while retaining or slightly improving on task accuracy. (\cref{sec:evaluation}). 
\end{itemize}

\section{Overview: \paradigmlong} \label{sec:overview}

In this section we first review how modern language models (LMs) are utilized and the challenges that arise from this.
Then, based on examples, we show how \paradigmlong (\paradigm) can overcome or simplify these challenges and outline the rest of the paper.
While our goal with \paradigm is to improve the usage of state-of-the-art large language models (LLMs), e.g. GPT \cite{radford2019language} variants, the size of the model does not change how \paradigm is employed, we thus utilize the acronym LM rather than the more common LLM in the remainder of this text.

\begin{wrapfigure}[5]{r}{0.4\textwidth}
\vspace{-0.6cm}
\centering
\begin{lstlisting}
"She sells seashells by the seashore."
["She", "@\textvisiblespace{}@sells", "@\textvisiblespace{}@seas", "hell", "s",
"@\textvisiblespace{}@by", "@\textvisiblespace{}@the", "@\textvisiblespace{}@se", "ash", "ore", "."]
\end{lstlisting}
\vspace{-0.2cm}
\caption{Tokenization of a sentence.} \label{fig:tokenization}
\end{wrapfigure}

\subsection{Background: (Large) Language Models} \label{sec:background}
Current language models \citep{VaswaniSPUJGKP17,radford2019language,BrownMRSKDNSSAA20} operate on a vocabulary \voc of (sub-word) tokens.
\cref{fig:tokenization} shows this for a simple example, where we see that common words have their own token (even with a space in front), while more rare words are split into multiple tokens.
Similar to formal languages we let $\voc^*$ denote all possible sequences of tokens over \voc. Given an input sequence of words $\vw_1, \dots \vw_t$, a tokenizer then first maps the sequence of words to a sequence of tokens $\vt_1, \dots, \vt_k$, and then a language model $\vf: \voc^k \to \mathbb{R}^{|\voc|}$ predicts a score $\vz = \vf(\vt_1, \dots, \vt_k)$ for every possible next token. We treat the implementation of $\vf$ as a black box (it does not need to be a neural network), yet in practice most such models are variants of the Transformer architecture~\citep{VaswaniSPUJGKP17}.
Via the softmax function, the resulting scores $\vz$ can then be turned into a probability distribution over the vocabulary $\mathcal{V}$:
\begin{equation*}
\softmax(\vz)_i := \frac{\exp(z_i)}{\sum_j \exp(z_j)}.
\end{equation*}

\paragraph{Decoding}
Based on this, the language model $\vf$ is applied multiple times to produce a sequence $\vt_1, \dots, \vt_K$ for $K > k$.
When we want to pick the $(i+1)$-th token, $\softmax(\vf(\vt_1, \dots, \vt_i))$ gives a probability distribution over this next token.
Several ways of picking from this distribution have been discussed in the literature. Below we review a selection of the most popular ones. Each method is iterated until a special end-of-sequence-token \eos is predicted or another stopping criterion is met. This can be seen as sampling from a distribution over $\voc^*$, and thus, some of these methods can return multiple possible decodings:
\begin{itemize}
\item \textbf{Greedy decoding} (or \textbf{Argmax decoding}) picks the token with the highest probability at each turn and feeds it back into the model to predict the next one (this corresponds to a depth-first search of all possible decodings). 
Importantly, this decoding does not necessarily (and in practice very rarely) correspond to the decoding with the highest overall probability (obtained by multiplying all individual probabilities of selected tokens). As this determines just the most probable decoding. Overall, only one decoding is returned.
\item \textbf{Sampling}, treats the output \softmax{} distribution as a categorical distribution from which a next token can be sampled. With sampling, it is common to decode multiple, e.g., $n$,  outputs.
\item \textbf{Full decoding} enumerates all possible sequences to the end and 
picks the one with the highest probability. This corresponds to a breadth-first search of all possible decodings. However, such enumeration (even with optimizations) is prohibitively expensive.
\item \textbf{Beam search} picks the middle ground between greedy and full decoding. It maintains a set of $n$ beams at all times, each corresponding to a predicted sequence. For each sequence, it predicts a possible next token and again picks the top $n$ from the resulting $n |\voc|$ sequences. In the end, the top sequence from the $n$ resulting beams is picked.
\end{itemize}
For beam search and sampling, an additional parameter, the temperature $\tau \in \mathbb{R}^{>0}$, can be used to control the diversity of the output, by using $\softmax(\vz/\tau)$ rather than $\softmax(\vz)$. A higher $\tau$ leads to more diverse outputs, while a lower $\tau$ leads to more likely outputs.

\paragraph{Masked Decoding}
A particular case of decoding is if we can already rule out certain tokens at certain positions. This means we can simply ignore these tokens and perform decoding over the remaining set.
In such a case, we assume that we are given a mask $\vm \in \{0, 1\}^{|\voc|}$, where a $1$ denotes a viable token and a $0$ denotes a discarded one.
We can apply the decoding methods discussed above on $\vm \odot \softmax(\vz)$, where $\odot$ denotes element-wise multiplication. (Note that, to obtain correct probabilities again this vector needs to be scaled by $1/\sum_i (\vm \times \softmax(\vz))_i$.)
An extreme case of this occurs when asking the model yes/no questions or classification tasks (e.g., to "positive" or "negative"). There we only allow the model to respond with the respective word and thereby the corresponding tokens.
Another case where this is applied, is when decoding a formal language such as in code completion or synthesis, where only a subset of possible tokens can form a legal program according to a grammar.

\begin{wrapfigure}[7]{r}{0.38\textwidth}
\vspace{-2em}
\centering
\begin{lstlisting}
Translate English to French:
sea otter => loutre de mer
peppermint => menthe poivrée
plush giraffe => girafe peluche
cheese =>
\end{lstlisting}
\vspace{-0.35cm}
\caption{Example of few-shot prompting; originally presented in \citet{BrownMRSKDNSSAA20}.} \label{fig:few_shot}
\end{wrapfigure}

\paragraph{Few-Shot Prompting}
Few-shot prompting \citep{BrownMRSKDNSSAA20} refers to the idea that language models do not need to be specifically trained for a downstream task (e.g. classification, question answering, etc.). Rather, it is sufficient to train them on broad text-sequence prediction datasets (e.g., the pile \citep{pile}) and to provide context in the form of examples when invoking them.
We show an example of this in \cref{fig:few_shot}, where our goal is to translate "cheese" from English to French. To this end we provide several examples of successful translation pairs and then ask the LM to complete the pair for "cheese" in the same syntax, where we expect the model to predict the tokens forming \lstinline|fromage| followed by the end-of-sequence token. In this way, translation and other tasks can be reframed as simple sequence completion tasks, which makes LMs powerful multi-task reasoners.

\paragraph{Multi-Part Prompting}
Due to their powerful reasoning capabilities, LMs are no longer just used for simple prompt completion, but also as compositional reasoning engines integrated into larger programs. Recent work explores a range of LM programming schemes, including Iterated Decompositions \cite{reppert_iterated_2023}, meta prompting \cite{ReynoldsM21}, and tool use \cite{yao2022react,schick2023toolformer}. Other projects, like \texttt{langchain} \cite{langchain23} are more focused on the composition of multiple prompts that are used in sequence. Similarly, LM~cascades~\cite{dohan2022language} frame compositional LM use in a probabilistic programming context.

\subsection{Key Challenges} \label{sec:challenges}
In this section we identify three key challenges in LM utilization, before outlining in \cref{sec:lm_programming} how \paradigmlong and \tool can be used to overcome them.

\paragraph{Interaction} LM interaction during the decoding process still remains a challenge. Consider for example the approach from \citet{ReynoldsM21}, which discusses the idea of \emph{meta prompts}, where in order to obtain the answer to a particular question, a language model is first asked to expand the prompt, which is then fed again to the same model in order to obtain an answer. We show an example of this in \cref{fig:circumference}~(a).
There, the goal is to find an answer to the question "What is the circumference of the earth?". In meta prompting, we first ask the language model for the name of an expert regarding this question, and then ask how this expert would answer the question.
With current LM interfaces, one would input the first part of the prompt, manually invoke the LM to complete the sequence with the expert name, then extract the expert name from the LM output, enter it manually into the rest of the template, and again feed it to the LM to obtain the actual answer.
This current approach requires a large amount of manual interaction via an API, or even a human in the loop (HITL). Once a value is fixed, e.g., the expert name, the decoding algorithm will assume it to be a fixed part of the prompt and will not optimize it jointly with the rest of the answer. In the HITL setting this enables the user to manually try multiple expert names and pick their favorite respective query completions. However, it precludes automated joint optimization of all template parameters to maximize the overall likelihood, which may yield better results.

\begin{figure}
\begin{minipage}{0.5\textwidth}
{
\scriptsize
\textbf{(a) Manual Prompt}\\
\colorbox{my-blue}{What is the circumference of the earth?}\\[-0.2em]
\colorbox{my-blue}{I believe the best person to answer this question is}~\colorbox{my-orange}{\underline{\hspace{1.5cm}}}.\\[-0.2em]
\colorbox{my-green}{Indeed, \underline{\hspace{1.5cm}} addressed this question:}\\[1em]
}
{
\scriptsize
\colorbox{my-blue}{Prompt 1}\hfill\colorbox{my-orange}{LM completion}\hfill\colorbox{my-green}{Prompt 2}
}\\
\hrule
\vspace{0.4em}
{\scriptsize \textbf{(c) \tool query}}\\[-1.8em]
\begin{lstlisting}[xleftmargin=0em, framesep=0em]
What is the circumference of the earth? I believe
the best person to answer this question is [EXPERT]
Indeed, {EXPERT} addressed this question: [ANSWER]
\end{lstlisting}
\hrule
\vspace{0.4em}
{\scriptsize \textbf{(d) \tool constraint}}\\[-1.8em]
\begin{lstlisting}[xleftmargin=0em, framesep=0em]
len(words(EXPERT)) <= 3 and stop_at(EXPERT, ".")
\end{lstlisting}

\end{minipage}
\hfill
\vline
\hfill
\begin{minipage}{0.44\textwidth}
\footnotesize
\textbf{(b) GPT-2 completions after \colorbox{my-blue}{Prompt 1}:}
\begin{itemize}[leftmargin=*]
\item a physicist
\item an astronomer
\item a geologist
\item Neal deGrasse Tyson
\item William O'Malley, who has a PhD in Geodesy and is a professor at Colorado State University.
\item the person having the knowledge and answer will probably have to refer to the relevant geophysics book and equations derived from that theory.
\item a physicist, like Thomas Kugler at UC Irvine or one of the other physicists working with NASA \dots
\item a man named David
\item actually Mother Earth herself? 
\end{itemize}
\end{minipage}
\caption{Example of a meta prompt for the circumference of the earth and its scripted prompting counterpart.}
\label{fig:circumference}
\end{figure}

\paragraph{Constraints \& Token Representation}
Another issue of the example query in \cref{fig:circumference} arises when we consider the completions as shown in \cref{fig:circumference}~(b). Sometimes, LMs will digress during generation and produce long ongoing sequences of text. While some answers work well for substitution in the next part of the prompt, others produce awkward and clumsy sentences at least and wrong sentences at worst. This is particularly problematic, if the result of an LM should be processed by another computer system, which may only be able to handle a very specific output format. In practice this means that users actually have constraints regarding the generated text, which sometimes are violated, as the LM does not adhere to them naturally. Ideally, these constraints should be expressible in terms of human understandable concepts and logic, since users will reason in terms of words, sentences and entities, not on a token level like the LM. However, practical methods of constraining LMs in this way \cite{shin2021constrained, PoesiaP00SMG22} still involve a lot of manual implementation effort and model-level understanding of the decoding procedures, tokenization and vocabulary of the LM.

\paragraph{Efficiency and Cost} Lastly, efficiency and performance remain big challenges. While a lot of work went into making the inference step in modern LMs more efficient, they still require expensive, high-end GPUs to be run with reasonable performance. Because of this, many practical users resort to hosted models running in the cloud, some of which are even guarded behind paid APIs. For this reason, LM querying can become very expensive, both in a computational and a financial sense. When relying on Language Model Programming  and constraints however, new opportunities for optimization arise, as predefined behavior and a limitation of the search space can be exploited to reduce the number of times an LM has to be invoked. In this setting, the cost of validation, parsing and mask generation is negligible compared to the vast cost of even just a single LM call. 

\subsection{\paradigmlong in \tool} \label{sec:lm_programming}
Now we consider \paradigmlong instantiated via our implementation \tool, and how it can help overcome these challenges.
Shown in \cref{fig:circumference}~(c), we write the same query as before in \tool syntax (formally defined in \cref{sec:language}). Here, when we encounter the construction \lstinline{[VAR]}, everything before the variable is fed to the LM and the answer found via decoding is then assigned to the variable \lstinline{VAR}, while a variable name in braces just recalls previously defined variables. This greatly simplifies the prompt and removes the need for manual interaction. Additionally, it enables the use of decoding procedures that consider both the expert name and answer jointly (as discussed in \cref{sec:query}).

Further, to address the issue of long on-running sentences, \tool allows constraints on the variable parts of the LM interaction on an intuitive level, e.g. words and phrases. \cref{fig:circumference}~(d) shows the intuitive \tool syntax for this, also discussed formally later on.
Here, the constraints enforce that the decoded tokens for \lstinline{EXPERT} are at most three words and that decoding stops if the sequence ends in a "\lstinline{.}". While it is possible to specify a maximum length with current query APIs, they usually work directly on the (model-specific) token level and thus cannot be mapped 1-to-1 to longer sequences.
In contrast, \tool supports declarative high-level constraints that are eagerly enforced during decoding, using token level inference masks and partial evaluation semantics (\cref{sec:constraint_decoding}).

Overall, \paradigmlong generalizes and automates many multi-part prompting approaches as discussed in \cref{sec:background}. It improves over the manual interaction setting outlined in \cref{sec:challenges} in multiple ways:
In contrast to a user having to manually try multiple values for \lstinline{EXPERT} and then selecting the best one, \tool allows users to constrain the set of considered experts or impose other restrictions ahead-of-time, fully automating this selection process. Once developed and tested, an \tool query (and constraints) can then be applied to many different inputs in an unsupervised way, not requiring any HITL.
\tool constraints enforce that the output fits the prompt template and avoid failure cases such as running-on (e.g. \cref{fig:circumference}). However, more generally, constraints can also force a model to generate text, that unconstrained it would have never explored. When used correctly, this can even lead to an improvement of the observed downstream task accuracy. Lastly, \tool can also be notably more efficient than manual interaction, as often, constraints and scripting can be applied eagerly during decoding, not requiring multiple LM calls.

\ifbool{extended}{ \newpage
}

\begin{figure}
\begin{minipage}{0.20\textwidth}
\centering
{\scriptsize \textbf{\tool Program}}
\begin{lstlisting}
@\ebnfph{decoder}@ @\ebnfph{query}@ 
from @\ebnfph{model}@ 
[where @\ebnfph{cond}@]
[distribute @\ebnfph{dist}@]
\end{lstlisting}
\end{minipage}
\hfill
\vline
\hfill
\begin{minipage}{0.78\textwidth}
\centering
\begin{lstlisting}
@\ebnfph{decoder}@ ::=  argmax | beam(n=@\ebnfph{int}@) | sample(n=@\ebnfph{int}@)
\end{lstlisting}
\vspace{-0.55em}
\begin{lstlisting}
@\ebnfph{query}@ ::=  @\ebnfph{python\_statement}@+
\end{lstlisting}
\vspace{-0.55em}
\begin{lstlisting}
@\ebnfph{cond}@ ::=  @\ebnfph{cond}@  and @\ebnfph{cond}@ | @\ebnfph{cond}@  or @\ebnfph{cond}@ | not @\ebnfph{cond}@ | @\ebnfph{cond\_term}@
         |  @\ebnfph{cond\_term}@  @\ebnfph{cond\_op}@ @\ebnfph{cond\_term}@
@\ebnfph{cond\_term}@ ::= @\ebnfph{python\_expression}@
@\ebnfph{cond\_op}@ ::= < | > |  = | in 
\end{lstlisting}
\vspace{-0.55em}
\begin{lstlisting}
@\ebnfph{dist}@ ::= @\ebnfph{var}@ over @\ebnfph{python\_expression}@ 
\end{lstlisting}
\end{minipage}
\caption{Syntax of \tool. Brackets denote optional elements. Syntax is generally python based.}
\label{fig:syntax}
\end{figure}

\section{The LMQL Language} \label{sec:language}

\begin{wrapfigure}{r}{0.45\textwidth}
\vspace{-1em}
\centering
\begin{subfigure}[t]{0.42\textwidth}
\scriptsize
\noindent
\texttt{A list of things not to forget when travelling:}\\
\texttt{- \textcolor{my-full-blue}{sun screen}}\\
\texttt{- \textcolor{my-full-blue}{beach towel}}\\
\texttt{The most important of these is \textcolor{my-full-blue}{sun screen}.}
\caption{With \lstinline{argmax} decoding.} \label{fig:interaction-trace:argmax}
\end{subfigure}
\begin{subfigure}[t]{0.42\textwidth}
\vspace{2em}
\scriptsize
\noindent
\texttt{A list of things not to forget when travelling:}\\
\texttt{- \textcolor{my-full-blue}{keys}}\\
\texttt{- \textcolor{my-full-blue}{passport}}\\
\texttt{The most important of these is \textcolor{my-full-blue}{sun screen}.}\\
\\
\texttt{A list of things not to forget when travelling:}\\
\texttt{- \textcolor{my-full-blue}{watch}}\\
\texttt{- \textcolor{my-full-blue}{hat}}\\
\texttt{The most important of these is \textcolor{my-full-blue}{keys}.}
\caption{With \lstinline{sample(n=2)} decoding.} \label{fig:interaction-trace:sample}
\end{subfigure}
\caption{The interaction trace for the query from \cref{fig:query-example:list} for different decoding methods.}
\label{fig:interaction-trace}
\end{wrapfigure}

Here we provide a high-level explanation of the syntax of \tool, before discussing the runtime and language semantics next. For concrete examples, consider the \tool programs given in \cref{fig:query-example}.

The grammar of \tool is shown in \cref{fig:syntax}. 
An \tool program has 5 parts: the decoder, the actual query, the \lstinline|from| clause specifying the queried model, the \lstinline|where| clause specifying constraints, and lastly a \lstinline|distribution| instruction.
The decoder and model are both specified by strings, while query and constraints are given in python syntax. We now explain these components in detail:

The \ebnfph{query} block models the interaction with the model. Informally it can be thought of as the body of a python function subject to some restrictions and additions: i) We do not allow the declaration of inner functions (however, imports can be made), and ii) Each top-level string is treated as a direct query to an LM.
These query strings allow for two specially escaped subfields, similar to python f-strings\footnote{https://peps.python.org/pep-0498}: 1) \lstinline|"{varname}"| recalls the value of a variable from the current scope. And  2.), \lstinline|"[varname]"| represents a phrase that will be generated by the LM, also called \emph{hole}.
When the language model generates values for these holes, they will be subject to the constraints defined in the \lstinline|where| clause of the query. Under these constraints, the decoding procedure specified by \ebnfph{decoder} (disussed next) will be used.
Once decoding finishes, a corresponding variable will be created in the scope of the query program and assigned this value. If a variable with the same name already exists, it will be overwritten.

\ebnfph{decoder} denotes the decoding procedure employed by the \tool runtime when solving the query.
The presented version of \tool enables \lstinline{argmax}, \lstinline{sample} and \lstinline{beam}.
\lstinline{argmax} and \lstinline{sample} work as discussed in \cref{sec:background}. 
\lstinline{beam} however, denotes a novel procedure called \emph{scripted beam search} which performs beam search jointly over all holes and control flow. We discuss this further in \cref{sec:query}. Once completed, the result of a query program is comprised of a number of things: It contains the \emph{interaction trace}, that is, the whole text transcript of the \tool query with the answers of the LM in the holes substituted.
Further, the set of all hole variables is accessible, allowing clients to directly access specific parts of the LM response.
In case of \lstinline{sample} and \lstinline{beam}, the parameter $n$ specifies the number of samples or beams respectively. In this case, $n$ interaction traces with the respective variables will be returned.
In practice, we allow further parameters to the decoder to be specified, e.g. the temperature $\tau$, but omit them here in favor of readability.

To illustrate queries and decoding, consider \cref{fig:query-example:joke} which utilizes a query purely made from strings, and \cref{fig:query-example:list} which utilizes a combination of strings and control flow. An corresponding interaction trace is shown in \cref{fig:interaction-trace}. Note how in the program on the right, \lstinline|THING| is reassigned on each iteration of the loop, which is in line with the semantics of python.

\lstinline|from |\ebnfph{model} denotes which LM to use. In the presented implementation, \ebnfph{model} denotes a string identifying a text generation model from the popular Hugging Face Model repository \cite{huggingface_models} or a model available via the OpenAI API \cite{BrownMRSKDNSSAA20}, like the GPT \cite{BrownMRSKDNSSAA20} family. However, this can also be extended to other local models or API backends.

\lstinline|where |\ebnfph{condition} places constraints on the \lstinline|[varname]| hole variables, thereby constraining the language model in what it can generate. Constraints can be an arbitary conjunction or disjunction of \ebnfph{cond\_expr} which allow comparison ($<$, $>$, $=$) and membership (\lstinline{in}) checks between standard python expressions. Note that, as hole variables are added to the scope of the query program, they can also be referenced there. We allow any deterministic pure python function along with constants.
We distinguish, for reasons discussed in \cref{sec:constraint_decoding}
, built-in functions (discussed next) and user-defined functions, which also includes standard python built-ins.
If we invoke the LM multiple times for the same variable, i.e., \texttt{THING} in \cref{fig:query-example:list}, the constraints apply to all intermediate values.

Lastly, \lstinline|distribute |{}\ebnfph{var}{}\lstinline| in |{}\ebnfph{python\_expression} is an optional instruction that can be added to augment the returned result. Here, \ebnfph{var} \emph{must} refer to the last variable in the query and the python expression to a set (or other iterable). We will refer to this set as the support of the distribution.
\begin{wrapfigure}[8]{r}{0.46\textwidth}
\scriptsize
\begin{minipage}{0.43\textwidth}
\noindent
\texttt{A list of things not to forget when travelling:}\\
\texttt{- \textcolor{my-full-blue}{sun screen}}\\
\texttt{- \textcolor{my-full-blue}{beach towel}}\\
\texttt{The most important of these is $\begin{cases} \text{\textcolor{my-full-blue}{sun screen}} & 65\%\\ \text{\textcolor{my-full-blue}{beach towel}} & 35\% \end{cases}$.}
\end{minipage}
\hfill
\vspace{-1em}
\caption{Continuation of the example from \cref{fig:query-example:list} and \cref{fig:interaction-trace:argmax} when appending \lstinline{distribute ITEM over things} to the query.}
\label{fig:interaction-trace-distribute}
\end{wrapfigure}

For queries with \lstinline|distribution| clause, the interaction trace will only be evaluated up to prior to the last hole according to the specified decoding method. In addition to the holes decoded so far and the interaction trace, the last variable is not decoded, but rather the probability distribution over support. Thus, for every value in the support the likelihood of this output is evaluated.
\cref{fig:interaction-trace-distribute} shows this for the example from \cref{fig:query-example:list}. In this case, the interaction trace up to the brace is produced, as well as the distribution over the possible values after.
This is particularly useful to encode classification tasks such as sentiment analysis, where the downstream user is interested in the probability distribution over e.g. $\{$\texttt{POSITIVE}, \texttt{NEGATIVE}$\}$.

\begin{wrapfigure}[4]{r}{0.6\textwidth}
\vspace{-1.2em}
\scriptsize
\begin{tabular}{ll}
	$[w_1, \dots w_k] \leftarrow \text{\texttt{words}(\ebnfph{var})}$  &\textcolor{gray}{//splits \ebnfph{var} into words $w_1, \dots w_k$}\\
	$[s_1, \dots s_k] \leftarrow \text{\texttt{sentences}(\ebnfph{var})}$ &\textcolor{gray}{//splits \ebnfph{var} into sentences $s_1, \dots s_k$}\\
	$b \leftarrow \text{\texttt{stop\_at}(\ebnfph{var}, t)}$ \hspace{2em} &\textcolor{gray}{//indicates if \ebnfph{var} ends in token or string $t$}\\
\end{tabular}
\vspace{-1.4em}
\caption{Built-in functions of \tool.}
\label{fig:builtin}
\end{wrapfigure}

\subsection{Built-in Functions}
In the \lstinline|where| clause, we support a set of built-in functions in addition to standard python code. For instance, we implement the functions \lstinline{words}, \lstinline{sentences} that, given a string or token representation, convert it to the desired representation.
To enable users to explicitly define stopping criteria, we also provide \lstinline{stops_at}, which can be used to provide constraints within the \lstinline{where} clause. \lstinline{stops\_at(}{}\ebnfph{var}{}\lstinline{, }\ebnfph{str}{}\lstinline{)} expresses that when the variable \ebnfph{var} is decoded it should stop decoding of the variable when the specified phrase is encountered. For similar purposes we provide \lstinline{len} (not shown), which overloads its default python counterpart with the comparable functionality -- it returns the length of a string (or iterable). For these designated, built-in functions, we implement additional semantics, required for the efficient output validation and the generation of decoding masks, as discussed in \cref{sec:constraint_decoding}.
\ifbool{extended}{
We provide further implementation details in \cref{sec:implementation}.
}{}

\section{The LMQL runtime: Query Execution \& Decoding} \label{sec:query}

\begin{wrapfigure}[18]{r}{0.59\textwidth}
        \centering
    \vspace{-1.45em}
    \scalebox{0.9}{
    \begin{minipage}{0.64\textwidth}
        \centering
    \begin{algorithm}[H]
    \SetAlgoLined
    \LinesNumbered
    \DontPrintSemicolon
    \KwIn{string $s$, trace $u$, scope \scope,  language model $\vf$}
    \uIf{$s$ contains $[\text{\ebnfph{<varname>}}]$}
    {
        $s_{\text{pre}}, \text{varname}, s_{\text{post}} \gets \text{unpack}(s)$ \;
        \tcp*{e.g. "a [b] c" $\rightarrow$ "a ", "b", " c"}
        $u \leftarrow u s_{\text{pre}}$ \tcp*{append to trace}
        $v \leftarrow decode(\vf, $u$)$ \tcp*{use the LM for the hole} \label{alg:eval-string:decode}
        $\scope[\text{varname}] \leftarrow v$ \tcp*{updated scope}
        $u \leftarrow u v$ \tcp*{append to trace}
    }
    \uElseIf{$s$ contains $\{\text{\ebnfph{varname}}\}$}
    {
        $\text{varname}\gets \text{unpack}(s)$ \tcp*{e.g. "\{b\}" $\rightarrow$ "b"}
        $v \leftarrow \scope[\text{varname}]$ \tcp*{retrieve value from scope}
        $s \leftarrow \text{subs}(s, \text{varname}, v)$ \tcp*{replace placeholder with value}
        $u \leftarrow u s$ \tcp*{append to trace}
    }
    \Else{ 
        $u \leftarrow u s$ \tcp*{append to trace}
    }
    \caption{Evaluation of a top-level string $s$}
    \label{alg:eval-string}
    \end{algorithm}
    \end{minipage}
    }
\end{wrapfigure}

We now discuss how the \tool runtime executes a query.
To this end we consider the execution of the \ebnfph{query} as a python program.
In this execution we assume that, i) functions are pure and do not cause side effects, ii) functions are deterministic.
Ignoring the constraints in \lstinline|where| for now, the \ebnfph{query} is executed line-by-line like a regular python function with one difference:
At the beginning of the execution, the interaction trace $u \leftarrow \emptystring$ is initialized to the empty string $\emptystring$. Whenever a top-level string $s$ is encountered in the program execution, the procedure in \cref{alg:eval-string} is evoked.
If a hole \texttt{[\text{\ebnfph{varname}}]} is encountered, the string $s$ is split into the text
preceeding the hole $s_\text{pre}$, the variable name and the text after the hole $s_\text{post}$.
$s_\text{pre}$ is directly appended to $u$\footnote{As is common we use multiplication to denote string concatenation and write $uv$ to denote the concatenation of $u$ and $v$.}
, which is then used to $decode$ a sequence $v$ to fill the hole from the LM $\vf$. 
This string is then assigned to \ebnfph{varname} in the scope $\scope$ of the python program.
If $\{\text{\ebnfph{varname}}\}$ is encountered, the value of \ebnfph{varname} is retrieved from scope $\scope$ and the placeholder is replaced with the value. In all cases the string $s$ (with the decoded or substituted text replaced) is added to $u$. 
Note that, for simplicity in \cref{alg:eval-string} we assume that there is at most one hole or placeholder in a string $s$.
In practice we allow multiple. 
Formally this can be thought of as splitting $s$ into a list of strings and then applying \cref{alg:eval-string} to each resulting string.
We illustrate this execution model in \cref{fig:example-execution} where we list the evaluation steps of the first 7 lines of \cref{fig:query-example:list}.
The first two lines are directly appended to the interaction trace $u$, while the next two lines (emitted inside the for loop) contain holes, which invokes the $decode$ function, discussed next.
\paragraph{Decoding Algorithm}
When $decode$ is invoked, the decoding procedure declared at the top of the \tool program is utilized to generate a value for the placeholder. 
Decoding is usually stopped i) when an end-of-sequence token is produced, or ii) when no more tokens can be produced due to the given constraints (discussed in \cref{sec:constraint_decoding}). 
For decoding algorithms that just output a single possible sequence, such as \lstinline{argmax} or \lstinline{sample(n=1)} the straightforward combination of \cref{alg:eval-string} and standard decoding function denotes the full end-to-end decoding procedure. However, a particular case occurs if multiple results are produced, e.g., \lstinline{sample(n=}{}\ebnfph{int}{}\lstinline{)} produces $n$ possible interaction traces $u$.

In this case, we track $n$ parallel execution of the query program, where $decode$ acts non-deterministically. In practice, we execute all calls in lockstep, such that we can batch calls to the underlying model $\vf$ and therefore improve efficiency. In \cref{alg:eval-string} we assume that $decode$ returns an already de-tokenized string $v$, not a sequence of tokens. 

\ifbool{extended}{
\paragraph{\bslong}
With the decoder \lstinline{beam(n=}{}\ebnfph{int}{}\lstinline{)}, the query is executed similarly: When the first hole in the interaction is encountered, $n$ beams (with their estimated probabilities) are created and retained. Each beam then corresponds to an interaction trace $u$, for which the query function is executed independently. Note that each $u$ might cause different control flow.
\ifbool{extended}{
\begin{wrapfigure}[13]{r}{0.38\textwidth}
}{
\begin{wrapfigure}[12]{r}{0.38\textwidth}
}
\centering
\ifbool{extended}{
\vspace{0.75em}
}{
\vspace{-0.25em}
}
\scalebox{0.9}{
\begin{minipage}{0.42\textwidth}
\centering
\begin{algorithm}[H]
\SetAlgoLined
\LinesNumbered
\DontPrintSemicolon
\KwIn{trace $u$, scope $\sigma$, LM $f$}
\KwOut{decoded sequence $v$}
$v \leftarrow \emptystring$ \label{alg:decoding:init} \;
\While{\textbf{True}}{
$\vm \leftarrow \text{compute\_mask($u$, $\sigma$, $v$)}$\;
	\lIf{$\bigwedge_i (m_i = 0)$}{ \textbf{break} }
	$\vz \leftarrow \sfrac{1}{Z}\; \cdot \vm \odot \softmax(\vf(uv))$ \;%
$t \leftarrow \text{pick}(z)$ \;
\lIf{$t = \eos$}{ \textbf{break} }
$v \leftarrow vt$\;
}
\caption{Decoding}
\label{alg:decoding}
\end{algorithm}
\end{minipage}
}
\end{wrapfigure}

Further, since we only consider the top $n$ beams at each step, we also only continue query execution for the top $n$ beams. Interaction traces that are discarded along the way, are pruned and not extended further. On termination, the overall query result corresponds to final top $n$ interaction traces.
}{

\paragraph{\bslong}
With the decoder \lstinline{beam(n=}{}\ebnfph{int}{}\lstinline{)}, the query is executed similarly: When the first hole in the interaction is encountered, $n$ beams (with their estimated probabilities) are created and retained. Each beam then corresponds to an interaction trace $u$, for which the query function is executed independently. Note that each $u$ might cause different control flow. Further, since we only consider the top $n$ beams at each step, we also only continue query execution for the top $n$ beams. Interaction traces that are discarded along the way, are pruned and not extended further. On termination, the overall query result corresponds to final top $n$ interaction traces.
}

\begin{figure}
\scriptsize
\centering
\begin{tabular}{lll}
\multicolumn{1}{c}{\textbf{line}} &
\multicolumn{1}{c}{\textbf{update}} &
\multicolumn{1}{c}{\textbf{state after update}}\\

1 &
& \begin{tabular}{@{}l@{}l@{}}
$u =$ &$\emptystring$\\
$g =$ &$\{ \}$
\end{tabular}\\[0.0em]
\\[-0.50em]\hline\\[-0.30em]
2 &
\begin{tabular}{l}
$s \leftarrow \text{\texttt{"A\textvisiblespace{}list\textvisiblespace{}of\textvisiblespace{}things\textvisiblespace{}not\textvisiblespace{}to\textvisiblespace{}forget\textvisiblespace{}when"}}$\\
$u \leftarrow us$
\end{tabular}
& \begin{tabular}{@{}l@{}l@{}}
$u = $ & $\text{\texttt{"A\textvisiblespace{}list\textvisiblespace{}of\textvisiblespace{}things\textvisiblespace{}not\textvisiblespace{}to\textvisiblespace{}forget\textvisiblespace{}when"}}$\\
$g =$ & $\{ \}$
\end{tabular}\\[0.0em]
\\[-0.50em]\hline\\[-0.30em]

3 &
\begin{tabular}{l}
$s \leftarrow \text{\texttt{"travelling:\textvisiblespace{}\textbackslash{}n"}}$\\
$u \leftarrow us$\\
\end{tabular}
& \begin{tabular}{@{}l@{}l@{}}
$u = $ & $\text{\texttt{"A\textvisiblespace{}list\textvisiblespace{}of\textvisiblespace{}things\textvisiblespace{}not\textvisiblespace{}to\textvisiblespace{}forget\textvisiblespace{}when travelling\textvisiblespace{}\textbackslash{}n"}}$\\
$g =$ & $\{ \}$
\end{tabular}\\[0em]
\\[-0.50em]\hline\\[-0.30em]

4, $i=0$ &
\begin{tabular}{l}
$s \leftarrow \text{\texttt{"-\textvisiblespace{}[THING]\textbackslash{}n"}}$\\
$s_{\text{pre}}, \text{varname},  s_{\text{post}} \leftarrow \text{\texttt{"-\textvisiblespace{}"}}, \text{\texttt{THING}}, \text{\texttt{\textbackslash{}n}}$\\
$u \leftarrow u s_{\text{pre}}$\\
$v \leftarrow \text{\textcolor{my-full-blue}{\texttt{"sun\textvisiblespace{}screen"}}} = \texttt{decode($\vf$, $u$)}$\\
$u \leftarrow u v s_{\text{post}}$\\
$g[\text{varname}] \leftarrow v$
\end{tabular}
& \begin{tabular}{@{}l@{}l@{}}
$u = $ & $\text{\texttt{"A\textvisiblespace{}list\textvisiblespace{}of\textvisiblespace{}things\textvisiblespace{}not\textvisiblespace{}to\textvisiblespace{}forget\textvisiblespace{}when travelling\textvisiblespace{}\textbackslash{}n}}$\\
&$\text{\texttt{-\textvisiblespace{}sun\textvisiblespace{}screen\textbackslash{}n"}}$\\
	$g =$ & $\{i=0, \text{THING}=\text{\texttt{"sun\textvisiblespace{}screen"}}, $\\
	&  $\text{things}=[\text{\texttt{"sun\textvisiblespace{}screen"}}]\}$
\end{tabular}\\[0em]
\\[-0.50em]\hline\\[-0.30em]

4, $i=1$ &
\begin{tabular}{l}
$s \leftarrow \text{\texttt{"-\textvisiblespace{}[THING]\textbackslash{}n"}}$\\
$s_{\text{pre}}, \text{varname},  s_{\text{post}} \leftarrow \text{\texttt{"-\textvisiblespace{}"}}, \text{\texttt{THING}}, \text{\texttt{\textbackslash{}n}}$\\
$u \leftarrow u s_{\text{pre}}$\\
$v \leftarrow \text{\textcolor{my-full-blue}{\texttt{"beach\textvisiblespace{}towel"}}} = \texttt{decode($\vf$, $u$)}$\\
$u \leftarrow u v s_{\text{post}}$\\
$g[\text{varname}] \leftarrow v$
\end{tabular}
& \begin{tabular}{@{}l@{}l@{}}
$u = $ & $\text{\texttt{"A\textvisiblespace{}list\textvisiblespace{}of\textvisiblespace{}things\textvisiblespace{}not\textvisiblespace{}to\textvisiblespace{}forget\textvisiblespace{}when travelling\textvisiblespace{}\textbackslash{}n}}$\\
&$\text{\texttt{-\textvisiblespace{}sun\textvisiblespace{}screen\textbackslash{}n}}$\\
&$\text{\texttt{-\textvisiblespace{}beach\textvisiblespace{}towel\textbackslash{}n"}}$\\
$g =$ & $\{i=1, \text{THING}=\text{\texttt{"beach\textvisiblespace{}towel"}},$\\
	&  $\text{things}=[\text{\texttt{"sun\textvisiblespace{}screen"}}, \text{\texttt{"beach\textvisiblespace{}towel"}}]\}$
\end{tabular}\\[-0.5em]

\end{tabular}
\caption{Example execution of the first 7 lines in \cref{fig:query-example:list}. Text generated by the LM $\vf$ in \textcolor{my-full-blue}{blue}.} \label{fig:example-execution}
\end{figure}

\paragraph{Language Model Integration} As shown in our decoding algorithm, we do not impose any restrictions on language model $\vf$, apart from being able to access the resulting distribution over vocabulary tokens. As, fundamentally, this is the core interface of most language models, we can easily integrate them without further changes. In fact, we implement \cref{alg:decoding} based on the \texttt{generate()} function from the HuggingFace \texttt{transformers} \cite{wolf2020transformers} package. Because of this, \tool{} already supports the large number of LMs available in the HuggingFace Model repository \cite{huggingface_models}.

\paragraph{Performance Considerations}
For large $n$ the execution of query code for multiple samples or beams can potentially be expensive, especially if compute-intensive functions are invoked on top of the LM output. However, as we assume functions to be pure and deterministic, results can be cached based on the function arguments, therefore greatly decreasing the total number of required function invocations. 
We also note that LMQL can evaluate constraints, control flow and compute token masks in parallel with the LM predicting its next token distribution. Only then, token masks need to be applied to continue text generation. This means that the LMQL runtime can run in lock-step with the LM, without incurring additional latency. One exception from this is if query execution itself entails blocking and interactive behavior such as web requests. In these cases, however, the latency is inherent due to the dependency on external systems, not a property of LMQL. In case the LM runs remotely on a different machine, LMQL additionally employs speculative LM prediction with asynchronous token masking, which helps to lower latency induced by network communication.

\paragraph{Decoding Internals}
\cref{alg:decoding} shows the internals of a decoding procedure (\text{decode} in \cref{alg:eval-string}) for a single sample or beam.
Here, the goal is to build up the string $v$, initialized to the empty string $\emptystring$ in \cref{alg:decoding:init}, by appending tokens $t$ to it. For each new token we compute a mask $\vm$ over the vocabulary, which only allows tokens that result in legal sequences, e.g., those that satisfy our \lstinline|where| constraints. If we can not produce any further tokens (i.e., $\bigwedge_i m_i = 0$) we stop the decoding procedure. Otherwise, we re-normalize $\vm \odot \vz$ into a probability distribution, i.e. a vector where entries add up to 1, by dividing it by $Z = \sum_i (\vm \odot \vz)_i$.
The function pick depends on the exact decoding algorithm (e.g. \lstinline|argmax|, \lstinline|sample|, \lstinline|beam|) and is used to pick a token $t$ from the distribution.
If we obtain an end-of-sequence \eos token we stop.
If we return early because no legal tokens are available, we are unable to find a response to the query that fulfils the constraints. If we return at \eos, we found a legal decoding. Next, we discuss how to compute the mask $\vm$, such that the specified constraints can be enforced during decoding.

\newcommand{\final}{\text{\scshape{Final}}}
\newcommand{\follow}{\text{\scshape{Follow}}}

\section{Validation and Constraint Decoding} \label{sec:constraint_decoding}

In this section we show how our decoding procedure can be extended to handle validation and constrained decoding. In particular, we discuss how the constraints from the \lstinline{where} clause can be used to automatically and efficiently find decoding masks for each step of decoding. Our main contribution to this end is a purpose-designed, eager execution model that supports partial evaluation and lookahead. To motivate this, we first discuss a naive solution and then introduce the idea of \textit{final semantics} and \cd{}s, the two abstractions at the core of our evaluation model. 

\paragraph{Naive Approach} 

\begin{wrapfigure}[15]{r}{0.50\textwidth}
	\vspace{-1.2em}
	\centering
	\scalebox{0.9}{
	\begin{minipage}{0.56\textwidth}
	\centering
	\begin{algorithm}[H]
	\SetAlgoLined
	\LinesNumbered
	\DontPrintSemicolon
	\KwIn{trace $u$, scope $\scope$, language model $f$}
	\KwOut{decoded sequence $v$}
	\SetKwRepeat{Do}{do}{while}
	\SetKwProg{Fn}{Function}{}{}
	\Fn{decode\_step($f$, $u$, $v$)}{
	$\vz \leftarrow \softmax(\vf(uv))$ \;
	$\vm \leftarrow \mathbf{1}^{|\voc|}$ \;
	\Do{$\bigvee_i m_i = 1$}{
		$t \leftarrow \text{pick}( \sfrac{1}{Z} \cdot \vm \odot \vz)$ \;
	\lIf{$t \neq \eos$}{decode\_step($u$, $v$, $vt$)}
	\lElseIf{$t = \eos \land \text{check}(u, vt)$}{\KwRet $v$}
	\lElse { $\vm[t] \leftarrow 0$ }
	}
	}
	decode\_step($f$, $u$, \emptystring) \;
	\caption{Naive Decoding with Constraints}
	\label{alg:naivedecoding}
	\end{algorithm}
	\end{minipage}
	}
\end{wrapfigure}

We first consider a naive approach to constrained decoding, outlined in \cref{alg:naivedecoding}.
Here, similar to \cref{alg:decoding}, we start with an empty string $v$ and append tokens. However, we don't assume a function compute\_mask and thus apply a backtracking-based approach, where we generate sequences up to the \eos token and then check if $uv$ satisfies our constraints.
Checking the constraints, denoted as $check$, is easy as it just amounts to the evaluation of an expression.

Note that here we assume that $uv$ is sufficient to check the constraints, at least up to the hole corresponding to $v$.
If this is not possible, we would need to perform the generation sequence for the sequence of all holes, advancing to the next one, once \eos is produced, but potentially backtracking over all, if validation fails at some point later on.

This strategy leads to multiple problems:
First, navigating the search space of sequences using backtracking is computationally expensive, especially when considering that the search space of LMs (even when trained well), is still a combinatorial explosion due to the many likely continuations of any given sequence. Second, querying the LM can be very expensive. State-of-the-art models often require high-end GPUs or are only available as API-gated, paid services. Thus, every token that is generated and later dismissed incurs a significant computational or financial cost.

With this in mind, we implement eager, partial evaluation semantics that model not only whether or not an expression holds, but also whether the expression can be guaranteed to never hold for any possible continuation of the currently-generated sequence. This allows us to terminate early if validation already provides a definitive result. Further, our semantics enable us to automatically compute a subset of next tokens that are guaranteed to violate the expression. Using this token set, we can effectively prune the search space of an LM and prevent the costly generation of invalid sequences before they are even generated.

\subsection{Partial Evaluation}
\begin{table}
    \centering
    \renewcommand*{\arraystretch}{1.1}
    \caption{Evaluation rules for $\final$ semantics for the core operators of \tool.} \label{tab:final}
    \vspace{-0.75em}
    \scriptsize
    	\hfill
        \begin{tabular}{rl}
	    \textbf{expression} & \textbf{\scshape Final}$[\cdot\;;\sigma]$\\[0.8em]

            \ebnfph{const} & $\fin$\\
    
            python variable \ebnfph{pyvar} & $\var$\\
            previous hole \ebnfph{var} & $\fin$ \\
            current var \ebnfph{var} & \inc\\
		future hole \ebnfph{var} & $\inc$\\[1.2em]

            \texttt{words($v$)} & $\final[v]$\\
            \texttt{sentences($v$)} & $\final[v]$\\
	    \texttt{len($v$)} & $\final[v]$\\[1em]

		number equality $n$ \texttt{==} $m$ & \hspace{0em}\makecell[l]{$\begin{cases}
                \fin & \text{if } \final[n] = \fin\\ 
		    &\;\land\; \final[m] = \fin\\
		\var & \text{else}\end{cases}$\\
            }\\[1em]

            string equality $x$ \texttt{==} $y$ & \hspace{0em}\makecell[l]{$\begin{cases}
                \fin & \text{if } \final[x] = \fin \\
		             &\;\land\; \final[y] = \fin\\
                \fin & \exists i \bullet x[i] \neq y[i]\\
                     &\;\land\; \final[x] \neq \var\\
                     &\;\land\; \final[y] \neq \var\\
                \var & \text{else}\end{cases}$\\
            }\\[1em]

            function \texttt{fn($\tau_1, \dots, \tau_k$)}
                 &
                \makecell[l]{$\begin{cases}\fin & \text{if } \bigwedge_{i=1}^k a(\tau_i) = \fin \\ \var  & \text{else}\end{cases}$\\
                }\\[1em]
 
        \end{tabular}
	\hspace{0em}
	\hfill
	\vline
	\hfill
        \begin{tabular}{rl}
	    \textbf{expression} & \textbf{\scshape Final}$[\cdot\;;\sigma]$\\[0.8em]

            \texttt{stop\_at(var, $s$)} & \makecell[l]{$\begin{cases}
		    \fin & \text{if } \eval{var}$\texttt{.endswith}$(s)\\
		    & \;\wedge\; \final[var] = \inc\\ 
                \var & \text{else}\end{cases}$\\
            }\\[1em]

            \makecell[r]{$x$ \texttt{in} $s$\\for strings $x,s$} & \makecell[l]{$\begin{cases}
                \fin & \text{if $x$ \texttt{in} $s$} \land \final[x] = \fin \\
                     & \;\land\; \final[s] = \inc \\ 
                \var & \text{else}\end{cases}$\\
            }\\[1em]
            \makecell[r]{$e$ \texttt{in} $l$\\for string $e$, set $l$} & \makecell[l]{$\begin{cases}
                \fin & \text{if } \not\exists i \in l \bullet i\texttt{.startswith(e)}\\
                     &\; \land \; \final[x] \in \{\inc,\fin\}\\
                     &\; \land \; \final[l] = \fin\\ 
                \var & \text{else}\end{cases}$\\
            }\\[1em]
             $x$ \texttt{<} $y$ & \makecell[l]{$\begin{cases}
		    \fin & \text{if } $x$ < $y$ \land \final[x] \in \{\dec, \fin\}\\
		    &\;\land \;\final[y] \in \{\inc, \fin\}\\
                \var & \text{else}\end{cases}$\\
            }\\[1em]

             $a$ \texttt{and} $b$ & \makecell[l]{$\begin{cases}
		    \fin & \text{if } \exists v \in \{a, b\} \bullet \eval{v}^F = \fin(\bot)\\
		     \fin & \text{if } \forall v \in \{a, b\} \bullet \eval{v}^F = \fin(\top)\\
                \var & \text{else}\end{cases}$\\
            }\\[1em]
             $a$ \texttt{or} $b$ & \makecell[l]{$\begin{cases}
		    \fin & \text{if } \exists v \in \{a, b\} \bullet \eval{v}^F = \fin(\top)\\
		     \fin & \text{if } \forall v \in \{a, b\} \bullet \eval{v}^F = \fin(\bot)\\
                \var & \text{else}\end{cases}$\\
            }\\[1em]
            \texttt{not $a$  } & $\final[a]$\\

        \end{tabular}
	\hfill\\
\vspace{-1.5em}
\end{table}

Given some expression $e$ occurring in the \lstinline|where| condition, some interaction trace $u$ and some global scope \scope, we define the evaluation semantics of $\eval{e}$ on multiple levels:

\paragraph{Value Semantics} First, we interpret $e$ on a value level, meaning we define $\eval{e}$ as the value of evaluating $e$ as a python expression, given the variable values assigned in $\sigma$. 

\paragraph{Final Semantics} In addition to value semantics, we define so-called \emph{final semantics} as a function $\final[e; \scope]$. The function $\final$ annotates each computed value with one of the annotators $\mathcal{A} = \{ \fin, \var, \inc, \dec \}$. Depending on the annotator, the value of an expression $e$, as decoding progresses is either considered $\fin$ (it will retain a fixed value), $\var$ (its value may still change), $\inc$ (its value will monotonically increase) or $\dec$ (its value will monotonically decrease). For the latter two, we consider monotonicity both in a numerical sense and in a set theoretic sense (e.g. growing sets, append-only strings). Based on this, $\final$ can be computed by applying it recursively to the intermediate results of a top-level expression $e$, as defined by the rules in \cref{tab:final}. 

\paragraph{Notation} In the following, we use the short-hand notation $\final[e]$ instead of $\final[e; \scope]$, as we assume that the scope is always the global scope. Further, we will sometimes refer to value and final semantics jointly, i.e., we will denote the value of an expression $e$ as $\eval{e} = v$ and $\final[e] = \fin$, simply as $\eval{v}^F = \fin(v)$. For boolean expressions we let $\top$ denote \lstinline|True| and $\bot$ \lstinline|False|.

\paragraph{Application} Using $\final$, we can evaluate \lstinline|where| constraints, even on outputs that are only partially available, i.e. a currently generating sequence. For this, we evaluate all (sub-)expressions, as far as possible. For expressions that depend on future hole values, we set their result to \texttt{None} and define all other operators to be tolerant of that. For instance, given some validation constraints $a \wedge b$, where $b$ cannot be determined yet, we can evaluate $a$ and return \texttt{False} if $a$ evaluates to $\fin(\bot)$. This is possible, as $\fin$ indicates that no matter the value of $b$, $a$ will always evaluate to $\bot$, even as more tokens of the generated sequence are revealed. 

\paragraph{Eager Validation} Final semantics provide an abstraction that enables us to implement more aggressive short-circuiting over validation conditions. These can be executed on each new token rather than waiting for the entire sequence to be generated. Using this, validation can be applied more eagerly, detecting invalid sequences before they are completed. However, final semantics do not help us to mask any next tokens in the decoding function. To enable this, we additionally introduce a third level of evaluation semantics, which we call \emph{follow semantics}, discussed next.

\subsection{Generating Token Masks using \cd{}s}

Provided that we can now evaluate \lstinline|where| conditions eagerly on every new token, the task that remains is to construct a token mask, that allows us to soundly identify tokens that are guaranteed to violate the condition when chosen next by the $decode$ function. To this end, we introduce a novel abstraction called \cd{}s.
\paragraph{Follow Maps} A follow map is a function $\cd{}(u,t)$ that takes a partial interaction trace $u$ and a token $t$ as input, and approximates the future value of some expression during validation, given $ut$ is validated next. We implement \cd{}s for all supported operators in \tool, and show a subset of the rules in \cref{tab:follow}. As shown, per operation, only a few rules are required. Note that a \cd{} always also produces a final annotator, but we only show them if the standard rules from \cref{tab:final} do not apply. 
Based on this, we define a recursive $\text{\textbf{\follow}}[\text{\ebnfph{expr}}](u,t)$ operator that automatically constructs the \cd{} for a provided expression, considering the definitions in \cref{tab:follow} as its base cases. This is implemented by recursively applying case-wise composition to the follow maps of the respective sub-expressions. Using \follow, we obtain an all-encompassing follow map for the entire validation expression. By inspecting the sub-cases of the resulting \cd{}, we then identify tokens that are guaranteed to violate the expression, which allows us to generate a decoding mask.

\paragraph{Example} Assume that we have the constraint \lstinline|TEXT in ["Stephen Hawking"]|  and that we are currently decoding hole variable \lstinline|TEXT|.
So far it has been assigned the value \texttt{"Steph"}. Using the rules in \cref{tab:follow}, we can construct a \cd{}:
$$
\follow[\text{\lstinline|TEXT in ["Stephen Hawking"]|}](\text{\texttt{"Steph"}},t) = \begin{cases}
\fin(\top) & \text{if } t = \text{\texttt{"en Hawking"}} \\
\fin(\bot) & \text{else}
\end{cases}
$$
The \cd{} returns $\fin(\top)$ if the following sequences matches \texttt{"en Hawking"} and $\fin(\bot)$ otherwise. During decoding, this can be translated into a token mask, as we know that tokens other than prefixes of \texttt{"en Hawking"} will definitively ($\fin$) violate our constraint. To enforce this, we derive a mask vector $\vm$ that only allows possible first tokens of \texttt{"en Hawking"} to be generated.

\begin{table}
\centering
\renewcommand*{\arraystretch}{1.1}
\caption{$\cd{}$ for the core set of operators supported in \tool. Whenever the final semantics of follow values do not align with standard behavior, we explicitly include final annotations. $v$ denotes the currently generated stream of tokens directly or as included as suffix in other computed values. $\eval[v \leftarrow vt]{\;\cdot\;}$ denotes evaluation under an updated scope, where $v$ is extended by $t$.} \label{tab:follow}
\vspace{-0.75em}
\scriptsize
\begin{tabular}{rl}
	 \textbf{expression} & \textbf{\follow}$[\cdot](u,t)$\\[1em]

		\ebnfph{const} & $\eval{\text{\ebnfph{const}}}$ \\
		\makecell[r]{python variable\\ \ebnfph{pyvar}} & $\eval[v \leftarrow vt]{\text{\texttt{pyvar}}}$\\
		previous hole  \ebnfph{var} & $\eval{\text{\ebnfph{var}}}$ \\
		current var  $v$ & $\begin{cases}
			\fin[v] & \text{if } t = \eos \\ 
			\inc[vt] & \text{else}
		\end{cases}$ \\
		future hole  \ebnfph{var} & \texttt{None}  \\[1em]

		\texttt{words($v$)} & 
		\makecell[l]{
		$\begin{cases}
			\fin[w_1, \dots, w_k] & \text{if } t = \eos  \\
			\inc[w_1, \dots, w_k] & \text{if } t = \text{\textvisiblespace} \\
			\inc[w_1, \dots, w_kt]   & \text{else}\\
		\end{cases}$ \\
		\quad where $w_1, \dots, w_k \leftarrow \eval{\text{\texttt{words($v$)}}}$
		}\\

		\texttt{sentences($v$)} & 
		\makecell[l]{
		$\begin{cases}
			\fin[s_1, \dots, s_k] & \text{if } t = \eos  \\
			\inc[s_1, \dots, s_k, t] & \text{if } s_k\text{\texttt{.endswith(".")}} \\
			\inc[s_1, \dots, s_kt]   & \text{else}\\
		\end{cases}$ \\
		\quad where $s_1, \dots, s_k \leftarrow \eval{\text{\texttt{sentences($v$)}}}$
		}\\

		\makecell[r]{\texttt{len($v$)}} & \makecell[l]{
		$\begin{cases}
			$\texttt{len}$(v)  & \text{if } t = \eos\\
			$\texttt{len}$(v) + 1 & \text{else}\\
		\end{cases}$
		}\\
		\makecell[r]{\texttt{len($l$)}\\over list $l$}  & \makecell[l]{
		$len(\eval[v \leftarrow vt]{l})$
		}\\
\end{tabular}
	\hspace{-0em}
	\hfill
	\vline
	\hfill
\begin{tabular}{rll}
	\textbf{expression} & \textbf{\follow}$[\cdot](u,t)$\\[1em]
		\texttt{fn($\tau_1, \dots, \tau_k$)} &
		$\text{\texttt{fn($\eval[v \leftarrow vt]{\tau_1}, \dots, \eval[v \leftarrow vt]{\tau_k}$)}}$
		\\[1em]
		
		\texttt{stop\_at($var$, s)} &  \makecell[l]{
		$\begin{cases}
			\fin(b) & \text{if } b \wedge \final[var] = \inc\\ 
			\var[l] & \text{else} \\
		\end{cases}$ \\
		where $b = \eval{var}\text{\texttt{.endswith}}(s)$
		}\\
	
		\makecell[r]{\texttt{x in }$s$\\for string $s$\\and constant $x$} & \makecell[l]{
			$\begin{cases}
				\top & \text{if \texttt{x in s}} \vee \texttt{x in $t$}\\
				\bot & \text{else}\\
		\end{cases}$
		}\\
		\makecell[r]{\texttt{x in }$l$\\for constant list/set $l$} & \makecell[l]{
			$\begin{cases}
				\fin[\top] & \text{if } \text{\texttt{t in l}}\\
				\var[\bot] & \text{if } \exists e \in l \bullet\\
				&\quad\text{\texttt{e.startswith($vt$)}}\\
				\bot & \text{else}\\
			\end{cases}$
		}\\
		\texttt{x $<$ y} & $\eval[v \leftarrow vt]{x} < \eval[v \leftarrow vt]{y}$ \\
		string comp. \texttt{a == $v$} & \makecell[l]{
			$\begin{cases}
				\fin[\top] & \text{if } vt =\text{\texttt{a}}\\
				\var[\bot] & \text{if } \text{\texttt{a.startswith($vt$)}}\\
				\bot & \text{else}\\
			\end{cases}$
		}\\
		number comp. \texttt{x == y} & $\eval[v \leftarrow vt]{x} = \eval[v \leftarrow vt]{y}$ \\
		 \texttt{a and b} & $\eval[v \leftarrow vt]{x} \texttt{ and } \eval[v \leftarrow vt]{y}$ \\
		 \texttt{a or b} & $\eval[v \leftarrow vt]{x} \texttt{ or } \eval[v \leftarrow vt]{y}$ \\
		 \texttt{not a} & $\texttt{not }\eval[v \leftarrow vt]{x}$ \\

\end{tabular}
\vspace{-1em}
\end{table}

\paragraph{Subtokenization}
To determine the set of valid sub-tokens that align with a follow continuation like \texttt{"en Hawking"}, we have to consider that most sub-word vocabularies allow for more than one factorization of a provided string into subtokens. This means, to determine the set of valid prefixes, we have to scan the entire vocabulary for possible prefix tokens and include all of them in the token mask, to maintain full expressiveness when it comes to the concrete choice of sub-word tokens that are used to encode a valid continuation. Here, we can assume that \follow{} is only ever applied to program states, where all model-generated values align with sub-token boundaries, because validation is performed eagerly on each new token, enabling this kind of prefix matching.

\paragraph{Soundness} While a perfect next-token validator is desirable, this can be hard to achieve, especially with constraints that rely on forward references. For this reason, we do not require $\follow$ to return \cd{}s that mask out all tokens that will violate our constraints (i.e. \textit{completeness}). Instead, we focus on \textit{sound} approximation: Given some boolean \lstinline|where| condition $e$ and the currently decoded hole variable $v$ (cf. \cref{alg:eval-string}), we consider the $\follow$ operator to be sound if and only if:
\begin{equation}
\forall t \in \mathcal{V} \bullet (\follow[{e}])(u,t) = \fin(\bot) \Rightarrow \eval[v \leftarrow ut]{e} = \fin(\bot)
\label{eq:soundness-constraint}
\end{equation}
In other words, if the returned \cd{} indicates that the next token $t$ is guaranteed to violate the condition $e$, then the condition $e$ must evaluate to $\fin(\bot)$ when $t$ is picked in the next decoding step. While this potentially over-approximates the set of valid tokens, it guarantees that we will never mask out any tokens that may actually be valid. Note also, how we rely on final semantics, i.e. $\fin(\bot)$, to express that a token will lead to a definitive violation of our constraints, and not just a temporary one during generation.
While over-approximation enables soundness, it also implies that some constraints cannot be enforced eagerly. In these cases, \tool has to resort to backtracking to find a valid sequence. This limitation is in line with theoretical results, as token masking using follow maps is comparable to context-free parsing.

\paragraph{Brzozowski derivatives}
To provide another perspective on \cd{} soundness, consider Brzozowski derivatives \cite{brzozowski1964derivatives}:
For a language $S \in \Sigma^*$, i.e. a set of strings over the alphabet $\Sigma$, and prefix $u \in \Sigma^*$ the Brzozowski derivative $u^{-1}S = \{ v \in \Sigma^* \mid uv \in S \}$ denotes the set of postfixes such that the concatenation $uv \in S$. In our case we are interested in the possible sequences over the token vocabulary $\voc^*$. In particular, given some query $\mathcal{Q}$, we are interested in the subset $L_\mathcal{Q} \subseteq \voc*$, which we do not necessarily have in closed form, that contains all interaction traces that fulfill the constraints specified in $\text{\texttt{where}}_\mathcal{Q}$. If during an execution of $\mathcal{Q}$ we have a partial interaction trace $u$, then $u^{-1}L_\mathcal{Q}$ denotes all possible legal postfixes completing this interaction trace. Using this, we define the \emph{set of Brzozowski-admissible tokens} $T_\mathcal{Q} = \{t \in \voc \mid (ut)^{-1}L_\mathcal{Q}) \neq \emptyset \}$, which can be decoded in the next step such that legal continuations in $L_\mathcal{Q}$ exist , i.e. $T_\mathcal{Q}$ describes the set of legal tokens for the next decoding step, thus forming a decoding mask $M$. 

Given these definitions, the \cd{} and the \follow{} operator satisfy the following theorem:

\begin{theorem}
	\textit{(Brzozowski Soundness)} Given a query $\mathcal{Q}$, partial interaction trace $u$, and the corresponding set of allowed tokens $M := \{t \in \mathcal{V} \;|\; \follow[\text{\texttt{where}}_\mathcal{Q}](u, t) \neq \fin(\bot)\}$, it holds that $T_\mathcal{Q} \subseteq M$, where $T_\mathcal{Q}$ is the set of Brzozowski-admissible tokens.
	\label{thm:brzozowski}
\end{theorem}
\begin{proof} \textit{(Brzozowski Soundness)}

	\begin{enumerate}
		\item By definition, we get the following: \begin{enumerate}
            \item $T_\mathcal{Q} \subseteq \mathcal{V}$, since we operate with limited vocabulary $\mathcal{V}$.
            \item Inverting the masking condition, we get $M = \mathcal{V} \setminus M^{-1}$ with the set of disallowed tokens $M^{-1} = \{t \in \mathcal{V} \;|\; \follow[\text{\texttt{where}}_\mathcal{Q}](u, t) = \fin(\bot)\}$
            \item Now, if we establish $T_\mathcal{Q} \cap M^{-1} = \emptyset$ (\textasteriskcentered), we can derive Brzozowski soundness as follows: \begin{center}$T_\mathcal{Q} \stackrel{(*)}{=} T_\mathcal{Q} \setminus M^{-1} \stackrel{(a)}{\subseteq} \mathcal{V} \setminus M^{-1} \stackrel{(b)}{=} M$ \;i.e.\; $T_\mathcal{Q} \subseteq M$\end{center}
            \item For $T_\mathcal{Q} \subseteq M$, it thus suffices to show (\textasteriskcentered), i.e. that no disallowed token in $M^{-1}$ is in $T_\mathcal{Q}$:  $\forall t \in \mathcal{V} \bullet t \in M^{-1} \implies t \notin T_\mathcal{Q}$.
        \end{enumerate}
        \item Now we prove (\textasteriskcentered): For any disallowed $t$ we know that $\follow[\text{\texttt{where}}_\mathcal{Q}](u, t) = \fin(\bot)$: \begin{itemize}
			\item Thus, for the current hole variable $v$, it holds that: $\eval[v \leftarrow ut]{\text{\texttt{where}}_\mathcal{Q}} = \fin(\bot)$.
			\item By final semantics, this means that there is no $p \in \mathcal{V}^*$ such that $\eval[v \leftarrow utp]{\text{\texttt{where}}_\mathcal{Q}} 
			\neq \bot$.
			\item By definition we know that $L_\mathcal{Q} := \{s \in \Sigma^* \;|\; \eval[\text{parse}(s)]{\text{\texttt{where}}_\mathcal{Q}} = \top\}$, where $\sigma[\text{parse}(s)]$ refers to the variable store, with variables set according to $\mathcal{Q}$ and interaction trace $s$.
			\item Therefore, we know that $utp \notin L_\mathcal{Q}$, which means that $tp \notin u^{-1}L_\mathcal{Q}$, i.e. $t \notin T_\mathcal{Q}$.
		\end{itemize}
		\item Overall, we therefore have shown that (\textasteriskcentered) holds, which implies via (1) that $T_\mathcal{Q} \subseteq M$. \qedhere
	\end{enumerate}
\end{proof}
This result is in line with \cref{eq:soundness-constraint}, and implies that \cd{}s will always allow, i.e. not mask out, any tokens that could still yield a legal decoding.

\section{Evaluation} \label{sec:evaluation}

Here, we evaluate the effectiveness of \tool as a language as well as a tool for prompt engineers. We evaluate \tool in three different case studies, encompassing a wide range of prompting scenarios.

\paragraph{Research Questions and Setup}

We focus our evaluation on three core questions:

\begin{itemize}
    \item \textbf{Expressiveness} Can we easily implement common and advanced prompting techniques with simple and concise query logic, especially in the case of interactive prompting?
    \item \textbf{Performance} Can \tool be used to effectively lower the required number of model queries and thereby lower the implied computational or API-related cost of using LMs?
    \item \textbf{Accuracy} Does LMQL's constrained decoding affect task accuracy of LMs when evaluated on standard benchmarks?
\end{itemize}

\paragraph{Baseline}

\tool{} provides a comparatively high-level interface, close to natural language prompting. Therefore, we evaluate \tool{} mainly as an alternative to other, existing high-level, text-based interfaces for Python, that are typically used to interact with LMs. 
\newcommand{\generate}{\texttt{generate()}}
More specifically, our baseline is a simple \generate{} API as e.g. provided by the HuggingFace Transformers package \cite{huggingface_generate}. 
\generate{} takes a string as input, for which it then generates a likely continuation sequence using a specified language models. This is a very accessible interface, but it does not support token level validation. We consider this as a reasonable baseline for \tool, as it reflects the current state of comparatively high-level LM APIs. For instance, \generate{} in the Transformers package does not support any token-level control beyond simple filter lists. The OpenAI API does allow logit masking, however, masks cannot be applied on a token-level, but only to the complete sequence.
These mechanisms are not capable of token-level validation and users have to handle parsing, validation and tokenization themselves. 
To reflect this, our \generate{} baseline is restricted to generating output chunk-wise, and doing parsing and validation manually. Multi-token constraints like \lstinline|THING in ["tube of sunscreen", "beach towel"]| or character-level length constraints cannot be enforced, as this requires token-level control. To enable stopping phrases, text is generated chunk-wise and when a stopping phrase is found, the output is truncated.

\paragraph{Datasets and Models} Our case studies address tasks relating to \textit{general and date understanding} \cite{srivastava2022beyond}, \textit{question answering} \cite{yang2018hotpotqa} and \textit{arithmetic math} \cite{cobbe2021training}.
With respect to the models, we rely on the publicly available open source model GPT-J 6B \cite{gpt-j} (6 billion parameters) and the more recent OPT-30B \cite{zhang2022opt} (30 billion parameters) model. Where GPT-J or OPT exceed our computational abilities, we rely on \texttt{gpt2-xl}\footnote{\url{https://huggingface.co/gpt2-xl}}, a 1.5B parameter version of GPT-2 \cite{radford2019language}.
We choose these models for evaluation as they are publicly available. This is crucial, because the \tool runtime requires integration with the decoding loop of a language model, which cannot be implemented efficiently with only limited high-level access.
The OpenAI API does not provide this kind of access, and we therefore evaluate on GPT-3 only in a limited fashion.

\paragraph{Metrics} To quantify performance, cost and usability characteristics of \tool, we consider a number of metrics:

\begin{itemize}
    \item \textbf{LOC} As a measure of conciseness we count the number of functional lines of code (LOC), i.e. excluding comments, empty lines, and fixed prompt parts (e.g. few-shot samples).
    \item \textbf{Number of Model Queries} We count the number of times the model $\vf$ is invoked for next-token prediction. This metric directly measures the computational cost of using a self-hosted LM, however, abstracts the computational cost of running the model itself.
    \item \textbf{Number of Decoder Calls} We also count the number of times a new decoding loop is started (a call to \generate{} in our baselines or an instance of the decoding \cref{alg:decoding} in \tool{}). We also count one Decoder Call per scored \lstinline|distribution| value, as this requires a new decoding loop to be started (in \tool{} and with \generate). This metric illustrates the API costs of an LM, as each decoder call will incur a cost, e.g. in terms of billing or latency.
    \item \textbf{Billable Tokens} Lastly, to model closely how API-gated models are billed, we count the number of tokens per Decoder Call, that are processed by the model as part of the prompt, plus the number of tokens that are generated. This metric is based on the billing mechanics of API-gated models like GPT-3. Based on Billable Tokens, we will make cost estimates, given the current token pricing of $\$0.02/1K$ tokens of the GPT-3 \texttt{davinci} model\footnote{\url{https://openai.com/api/pricing/}}. This highlights the potential savings if \tool could be used in place of standard high-level APIs.

\end{itemize}
We motivate this choice of performance metrics over pure runtime by the reality of using LMs in practice. Any reduction in the number of processed tokens will directly translate to a saving in cost, both with API-based models and when running a language model locally.

\paragraph{Experimental Setup} All language models are instantiated via the HuggingFace \texttt{transformers} library \cite{wolf2020transformers} with \texttt{pytorch} on the backend, using Nvidia A100 GPU with 40GB/80GB VRAM. 

\begin{figure}
    \centering
    \begin{lstlisting}[escapeinside={\@}{\@}]
argmax
    "Pick the odd word out: skirt, dress, pen, jacket.\n"
    "skirt is clothing, dress is clothing, pen is an object, jacket is clothing.\n"
    "So the odd one is pen.\n\n"
    "Pick the odd word out: Spain, France, German, England, Singapore.\n"
    "Spain is a country, France is a country, German is a language, @\dots@\n"
    "So the odd one is German.\n\n"
    "Pick the odd word out: {OPTIONS}\n"
    "[REASONING]"
    "[RESULT]"
from "EleutherAI/gpt-j-6B"
where
    not "\n" in REASONING and not "Pick" in REASONING and 
    stops_at(REASONING, "Pick the odd word") and stops_at(REASONING, "\n") and
    stops_at(REASONING, "So the odd one") and stops_at(REASONING, ".") and len(WORDS(REASONING)) < 40
distribute
    RESULT over OPTIONS.split(", ")
    \end{lstlisting}
\vspace{-1em}
\caption{\tool query implementing chain-of-thought prompting for the Odd One Out classification task.} 
\label{lst:eval-chain-of-thought}
\vspace{-1em}
\end{figure}

\subsection{Case Study 1: Chain-of-Thought Prompting}

We first consider multiple-choice question answering tasks: A LM is presented with a question $Q$ and a set of options $\mathcal{O} = \{O_1, \dots, O_n\}$. While direct prompting of a model to obtain the result as $argmax_\mathcal{O}\;P(O_i|Q)$ is possible, it is often not enough to reach good levels of performance. Further, the model's reasoning may not be clear and the resulting answers can appear quite arbitrary. \emph{Chain-of-thought} prompting \cite{wei2022chain} aims to address this, by preceding the actual question with few-shot samples that demonstrate how to arrive at a correct answer through a multi-step reasoning process. By priming the model in this way, it is more likely to produce a similar chain of thoughts, eventually leading up to the correct answer for a new question.
For this case study we implement queries for two task: The general knowledge reasoning task \emph{Odd One Out} and the \emph{Date Understanding} task, both included in the recent BIG benchmark collection \cite{srivastava2022beyond}.

\paragraph{Query and Results} We implement chain-of-thought reasoning in \tool as shown in \cref{lst:eval-chain-of-thought}. The prompt clause contains two few-shot examples with reasoning steps. We provide the comma-separated list of words of the Odd One Out task as query argument \lstinline{OPTIONS} when iterating over the dataset. The first hole variable generated by the model is \lstinline{REASONING}. We constrain the \lstinline{REASONING} variable in multiple ways, including a maximum number of words and several stopping conditions. Further, we disallow the use of \texttt{"Pick"} and the newline character, to prevent the model from digressing or skipping the reasoning steps alltogether. For decoding, we rely on \lstinline{argmax} which provides us with the greedily-determined most likely answer. 

\begin{table}
    \footnotesize
\footnotesize
    \caption{Average performance statistics (over queries) for constrained \tool chain-of-thought decoding compared with standard chunk-wise decoding for the Odd One Out and Date Understanding datasets.}
    \vspace{-1em}
    \begin{tabular}{l|rrrr|rrrr}
	    & \multicolumn{4}{c|}{\textbf{GPT-J-6B\cite{gpt-j}}} & \multicolumn{4}{c}{\textbf{OPT-30B \cite{zhang2022opt}}} \\
        & \textbf{\makecell*{Standard\\ Decoding}} & \textbf{\tool} & \textbf{$\Delta$} & \textbf{\makecell*{Est. Cost\\Savings}} & \textbf{\makecell*{Standard\\ Decoding}} & \textbf{\tool} & \textbf{$\Delta$} & \textbf{\makecell*{Est. Cost\\Savings}} \\
        \toprule
        \textit{Odd One Out} &&&&&&&\\
        Accuracy         &   
            33.33\% &   
            \textbf{34.52}\% &    
            1.19\% & 
            &
            \textbf{34.52}\% &
            \textbf{34.52}\% &
            0.00\% &
            \\
        Decoder Calls &     
            7.96 &
            \textbf{5.96} &  
            -25.11\% &
            & 
            7.96 & 
            \textbf{5.96} & 
            -25.11\% & \\
        Model Queries   &    
            73.04 &
            \textbf{41.51} & 
            -43.16\% &
            &
            73.04 &
            \textbf{40.70} &
            -44.27\% &
            \\
        Billable Tokens  &   
            1178.71 &   
            \textbf{861.32} &  
            -26.93\% &  
            0.63\textcent\raise-0.1ex\hbox{\tiny/query} &
            1173.21&
            \textbf{856.17} &
            -27.02\% &
            0.63\textcent\raise-0.1ex\hbox{\tiny/query}
            \\
        \midrule
        \textit{\makecell[l]{Date Understanding}} &&&&&&&\\
        Accuracy         &    
            \textbf{22.89}\% &
            \textbf{22.89}\% &
            0.00\% &
            &
            \textbf{29.16}\% &
            \textbf{29.16}\% &
            0.00\% &
            \\
        Decoder Calls &      
            9.84 &      
            \textbf{6.84} &  
            -30.47\% &
            &
            9.84 &
            \textbf{6.84} &
            -30.47\% &
            \\
        Model Queries   &     
            103.38 &     
            \textbf{57.26} &   
            -44.61\% &
            &
            103.38 &
            \textbf{57.00} &
            -44.86\% &
            \\
        Billable Tokens  &   
            4131.28 &   
            \textbf{2844.90} &  
            -31.14\% &  
            2.57\textcent\raise-0.1ex\hbox{\tiny/query} &
            4129.55 &
            \textbf{2842.93} &
            -31.16\% &
            2.57\textcent\raise-0.1ex\hbox{\tiny/query}
            \\
        \bottomrule
    \end{tabular}
    \label{tab:eval-chain-of-thought}
\vspace{-0.75em}
\end{table}

Lastly, we use the \lstinline{distribute} clause, to compute a probability distribution over the set of possible answers in $\mathcal{O}$, i.e. $P( \cdot | \text{\texttt{"\ebnfph{p}\ebnfph{q}\ebnfph{r}"}})$, which is conditioned on the concatenation of the few-shot samples \texttt{\ebnfph{p}}, the question \texttt{\ebnfph{q}} and the generated reasoning steps \texttt{\ebnfph{r}}.

Analogously to our \tool query, we implement the same prompting behavior with a \texttt{generate()}-based python program. As discussed, the baseline program employs similar stopping conditions for \texttt{REASONING} but does not encode token level constraints. We evaluate both programs on Odd One Out and Date Understanding with GPT-J/OPT-30B, and document the results in \cref{tab:eval-chain-of-thought}.

\paragraph{Results} Overall, we observe the same or improved accuracy for (constrained) \tool decoding when compared to Standard Decoding. Manual inspection reveals that the accuracy improvements on \textit{Odd One Out} can be traced back to the \lstinline|REASONING| variable: In \tool{}, the constraints shown in \cref{lst:eval-chain-of-thought} (e.g. word limit and disallowing  e.g. \texttt{"Pick"}) guide the model when generating \lstinline|REASONING|. In Standard Decoding, these constraints cannot be enforced due to the limitations of the \generate{} API, leading to a different \lstinline|REASONING| output. As in \textit{chain-of-though}, the final answer \lstinline|RESULT| is conditioned on the generated \lstinline|REASONING| steps (cf. task demonstrations in \cref{lst:eval-chain-of-thought}), \tool constraints lead to a different final answer and therefore impact accuracy.
With regards to efficiency, \tool reduces model queries and the number of billable tokens by up to $41\%$ and $31\%$ respectively. Overall, we observe a significant reduction in cost/compute, especially when considering that the \tool-based constrained decoding can achieve the same or better accuracy. 
We find that \tool{} reduces program size in LOC to $26\%$ ($34\%$ resp.) of the corresponding baseline implementation.

\paragraph{OpenAI GPT-3.5}
As a control experiment, we also run both Standard Decoding and the \tool{} queries on the GPT-3.5 model \texttt{text-davinci-003} (a limited integration of the OpenAI API is possible in LMQL). There, we also observe maintained accuracy for Odd One Out ($42.86$\%) and slightly improved performance on Date Understanding (Standard Decoding: $85.29$\%, \tool $86.10$\%).

\subsection{Case Study 2: Interactive Prompting}
\label{sec:eval-interactive-prompting}

Chain-of-thought prompting is an effective method to improve model understanding \cite{wei2022chain}. It can be used to extract knowledge from a model or generate new insights by multi-step reasoning. However, in some cases a model may not know about the required context information and external sources have to be consulted. For instance, for question answering the prompting scheme \texttt{ReAct}~\cite{yao2022react} proposes to augment chain-of-thought-based prompting with the ability for the model to interactively query external sources such as Wikipedia. As \tool supports loops, branches, and function calls in its prompt clause, it lends itself well to implementing these kinds of interactive prompting scenarios. By relying on control flow in the prompting clause of a query, we can interpret model results step-by-step and inject information from external sources. %

\paragraph{Query} To invoke external actions like Wikipedia lookups, \texttt{ReAct} relies on designated action phrases such as \texttt{Search} and \texttt{Finish}, that the LM can produce as needed. To implement this interactive behavior in \tool, we rely on a basic interpretation loop as shown in \cref{lst:eval-interactive-prompting}. The loop iterates over the model's output and interprets actions when applicable. Wikipedia lookups are implemented as calls to an external python utility. During branching and beam search with multiple hypotheses, the loop and corresponding lookup operations will automatically be issued as required during decoding. The loop terminates when the model generates a \texttt{Finish} action, storing the overall results of the query in the \texttt{SUBJECT} variable. To further guide the generation process, we constrain \texttt{MODE} to be in \;\{\texttt{Tho}, \texttt{Act}\}. Further, we implement simple stopping conditions for \texttt{THOUGHT} and \texttt{SUBJECT} to prevent the model from violating the \texttt{ReAct} reasoning pattern.

\begin{figure}[t]
\begin{lstlisting}[escapeinside={\@}{\@}, basicstyle=\ttfamily\scriptsize]
import wikipedia_utils
sample(no_repeat_ngram_size=3)
    "What is the elevation range for the area that the eastern sector of the Colorado orogeny extends into?"
    "Tho 1: I need to search Colorado orogeny, find the area that the eastern sector of the Colorado @\dots@\n"
    "Act 2: Search 'Colorado orogeny'\n"
    "Obs 2: The Colorado orogeny was an episode of mountain building (an orogeny) @\dots@\n"
    "Tho 3: It does not mention the eastern sector.  So I need to look up eastern sector.\n"
    @\dots@
    "Tho 4: High Plains rise in elevation from around 1,800 to 7,000 ft, so the answer is 1,800 to 7,000 ft."    
    "Act 5: Finish '1,800 to 7,000 ft'"
    "Where is Apple Computers headquartered?\n"
    for i in range(1024):
        "[MODE] {i}:"
        if MODE == "Tho": 
            "[THOUGHT] "
        elif MODE == "Act":
            " [ACTION] '[SUBJECT]\n"
            if ACTION == "Search": 
                result = wikipedia_utils.search(SUBJECT[:-1]) # cutting of the consumed '
                "Obs {i}: {result}\n"
            else:
                break # action must be FINISH
from "gpt2-xl"
where
    MODE in ["Tho", "Act"] and stops_at(THOUGHT, "\n") and
    ACTION in ["Search", "Finish"] and len(words(THOUGHT)) > 2 and
    stops_at(SUBJECT, "'") and not "Tho" in THOUGHT
\end{lstlisting}
\vspace{-1em}
\caption{\tool code for interactive \texttt{ReAct} \citep{yao2022react} prompting scheme for question answering.}
\label{lst:eval-interactive-prompting}
\vspace{-0.75em}
\end{figure}

\paragraph{Python Baseline} As a baseline for scripted interpretation, we implement a python program that supports the same \texttt{ReAct} prompting as the query in \cref{lst:eval-interactive-prompting}. To implement \tool's declarative parsing of \texttt{THOUGHT}, \texttt{SUBJECT}, and \texttt{ACTION}, we rely on built-in python functionality to parse and process the chunk-wise produced output. For this, we note that we have to resort to hand-crafted parsing logic, whereas in \tool we can simply rely on declarative predicates like \texttt{STOPS\_AT} and validation conditions in the where clause of the query. 
We note that the baseline implementation can only support \texttt{sample} and \texttt{argmax} decoding. Deeper integration, e.g. with beam search, is not easily realizably in python, as the prompting program must be capable of branching into multiple execution heads in accordance with the branching of decoding. In contrast, \tool supports this out-of-the-box. Lastly, in our baseline implementation, we have to invoke the model multiple times, each time generating a new chunk of output, parsing, and evaluating potential action phrases. For this, we have to choose the chunk size appropriately. We overview the implications of different choices for this parameter in \cref{fig:eval-react-chunk-size}.
For our comparison with \tool{}, we choose standard decoding with chunk size of 30, which minimizes the number of billable tokens, while not issuing exceedingly many model queries.

\paragraph{Results} To assess \tool performance benefits with interactive prompting workloads, we apply our \texttt{ReAct} implementations to a question answering task from the HotpotQA \cite{yang2018hotpotqa} dataset. We observe a significant reduction of decoder calls of up to 80\% when using \tool{} over standard decoding. This can be attributed to \tool{}'s ability to decode the whole sequence in one run, validating on-the-fly. Standard Decoding on the other hand has to decode the whole sequence in chunks, invoking \lstinline|generate()| at least as many times as interactions are required. Regarding the total number of model queries, we observe a reduction of at least $30\%$. For Billable Tokens, we observe an even stronger effect, where \tool saves up to $76\%$ of the tokens, leading to a significant saving in costs, i.e. $76\%$ fewer tokens or 5.2\textcent{}. Considering program size last, we implement \texttt{ReAct} in just $22$ LOC of \tool{}, which is $63\%$ fewer lines than in our python-based implementation.
\begin{figure}
    \centering
    \begin{tikzpicture}[baseline, font=\footnotesize, scale=0.9]
        \begin{scope}
		    \fill[color=black!2] (0,0) rectangle (3.41,2.25);
        \node at (1.7,2.7) {Model Queries};
        \begin{axis}[
            xlabel=Chunk Size,
            ylabel=,
            width=5cm,
	    height=3.8cm,
            xtick pos=bottom,ytick pos=left,
            xmin=18,
            ymin=0,
            ymax=300,
            name=firstplot,
            legend pos=south east,
            axis line style={draw=none},
            ]
        \addplot table {
            20 120
            30 150
            40 200
            50 250
        };
        \addplot+[mark=none,line width=1pt, dashed] table {
            20 95
            30 95
            40 95
            50 95
        };
        \legend{}
        \end{axis}
        \end{scope}
        \begin{scope}[xshift=3.9cm]
        \fill[color=black!2] (0,0) rectangle (3.41,2.25);
        \node at (1.7,2.7) {Decoder Calls};
        \begin{axis}[
            xlabel=Chunk Size,
            ylabel=,
            width=5cm,
	    height=3.8cm,
            xtick pos=bottom,ytick pos=left,
            xmin=18,
            ymin=0,
            name=firstplot,
            legend style={at={(2.73,0.5)},anchor=center,draw=none},
            legend cell align={left},
            axis line style={draw=none}
            ]
        \addplot table {
            20 6
            30 5
            40 5
            50 5
        };
        \addplot+[mark=none,line width=1pt, dashed] table {
            20 1
            30 1
            40 1
            50 1
        };
        \legend{Standard Decoding, LMQL Decoding}
        \end{axis}
        \end{scope}
        \begin{scope}[xshift=8.2cm]
        \fill[color=black!2] (0,0) rectangle (3.41,2.25);
        \node at (1.7,2.7) {Billable Tokens};
        \begin{axis}[
            xlabel=Chunk Size,
            ylabel=,
            width=5cm,
	    height=3.8cm,
            xtick pos=bottom,ytick pos=left,
            xmin=18,
            ymin=0,
            name=firstplot,
            legend pos=south east,
            axis line style={draw=none}
            ]
        \addplot table {
            20 4069
            30 3404
            40 3474
            50 3534
        };
        \addplot+[mark=none,line width=1pt, dashed] table {
            20 807
            30 807
            40 807
            50 807
        };
        \legend{}
        \end{axis}
        \end{scope}
    \end{tikzpicture}%
    \vspace{-0.25em}
    \caption{Comparing different chunk sizes used for the baseline implementation as compared to \tool{}, which does not require chunk-wise decoding. All results were measured for interactive \texttt{ReAct} prompting.}
    \label{fig:eval-react-chunk-size}
    \vspace{-1.5em}
\end{figure}

\subsection{Case Study 3: Arithmetic Reasoning}
\label{sec:eval-arithmetic}

\begin{wraptable}[10]{r}{0.4\textwidth}
    \vspace{-1.35em}
    \footnotesize
    \caption{Lines of Code (LOC) required to implement the baseline implementations and corresponding \tool queries.}
    \label{tab:eval-loc}
    \begin{tabular}{lcc}
    \toprule
    \textbf{Task} & \textbf{\makecell*{Python\\Baseline}} & \textbf{\tool} \\
    \midrule
    Odd One Out & 34 & 9 \\
    Date Understanding & 38 & 13 \\
    Arithmetic Reasoning & 59 & 22 \\
    \texttt{ReAct} & 78 & 18 \\
    \bottomrule
    \end{tabular}
\end{wraptable}

Lastly, we consider arithmetic reasoning. Existing work shows that LMs struggle with evaluating arithmetic expressions correctly \cite{wei2022chain}. While reasoning might be correct, mistakes in the concrete arithmetic calculations lead to an incorrect result \cite{wei2022chain,cobbe2021training}. This is exacerbated by the open-ended nature of math problems, where the result is not picked from a limited set of options, but can be any valid number. Recent works \cite{wei2022chain, cobbe2021training,andor2019giving} augment LMs with the ability to externally evaluate arithmetic expressions during generation.

\paragraph{Query} In \cref{fig:eval_arithmetics:query} we demonstrate arithmetic evaluation in \tool, relying on scripted prompting and constraints. The query decodes reasoning and calculations steps from the model, scanning for occurrences of \lstinline{"<<"}. 
Once it encounters such a sequence, it queries the model for the to-be-evaluated expression (e.g. \lstinline{1+2=?}), evaluates it using an external utility function, and passes back the result. 
\paragraph{Results} 
We applied our query, as well as a baseline program, to an arithmetic reasoning problem from the GSM8K dataset \cite{cobbe2021training}. As shown by the interaction trace in \cref{fig:eval_arithmetics:trace}, our \tool{} query detects and processes arithmetic expressions, as they occur in the model's output, leading up to the answer. 
The necessary query logic is comparatively basic, only requiring some text processing and a simple interpretation loop. Finally, by applying an \lstinline|int| constraint on \lstinline|RESULT|, we can enforce the final model's output to always be a valid integer. In this case, GPT-J 6B is not able to solve the problem correctly. However, the example still demonstrates that \tool{} can be used to implement on-the-fly arithmetic evaluation, aiding the model in solving the task. Collecting query statistics, we compare the two implementations in \cref{tab:eval-react-comparison}. For the baseline implementation (standard decoding), the number of decoder calls is determined by the number of arithmetic expressions in the model's output. For \tool{}, this has no impact, as arithmetic expressions is done on-the-fly. Overall this means that \tool{} only requires one decoder call, where the standard approach requires $7$. Further, we observe a significant reduction of 65\% in model queries and 85\% in billable tokens (saving 6.2\textcent{} per query with GPT-3 \lstinline|davinci|). The \tool{} implementation of arithmetic evaluation requires $18$ LOC, compared to $78$ LOC required for the python-based baseline.
\begin{table}[ht]	
    \caption{\tool constrained decoding compared to Standard Decoding in an interactive prompting scenario. 
    }
    \vspace{-0.75em}
    \footnotesize
    \begin{tabular}{lrrrr}
        \footnotesize
            & \textbf{Standard Decoding} & \textbf{\tool} & \textbf{$\Delta$} & \textbf{Est. Cost Savings} \\
    \toprule
   \textit{\texttt{ReAct} (Case Study 2)}\\
	    Decoder Calls &    5 &  \textbf{1} &  -80\% &           \\
	    Model Queries &    150 &    \textbf{95} &  -36.67\% &           \\
	    Billable Tokens & 3,404 & \textbf{807} & -76.29\% & 5.2\textcent\raise-0.1ex\hbox{\tiny/query}\\
        \midrule
    \textit{Arithmetic Evaluation (Case Study 3)}\\
	    Decoder Calls &              7 &      \textbf{1} &  -85.71\% &    \\                                         
	    Model Queries   &            210 &     \textbf{71} &  -66.19\%  & \\
	    Billable Tokens  &           3,649 &    \textbf{550} &  -84.93\% & 6.2\textcent\raise-0.1ex\hbox{\tiny/query}\\
    \bottomrule
    \end{tabular}
    \label{tab:eval-react-comparison}
\vspace{-0.75em}
\end{table}
\begin{figure}[h]
    \centering
    \begin{subfigure}[t]{0.49\textwidth}
    \begin{lstlisting}[escapechar=!]
argmax(distribution_batch_size=1, max_length=2048)
    "!\ebnfph{few-shot examples}!"
    "Q: {QUESTION}\n"
    "A: Let's think step by step.\n"
    for i in range(1024):
        "[REASON_OR_CALC]"
        if REASON_OR_CALC.endswith("<<"):
            " [EXPR] "
            result = calculator.run(EXPR)
            " {result} >> "
        elif REASON_OR_CALC.endswith("So the answer"):
            break
    " is [RESULT]"
from "EleutherAI/gpt-j-6B"
where
    int(RESULT) and
    stops_at(REASON_OR_CALC, "<<") and
    stops_at(EXPR, "=") and 
    stops_at(REASON_OR_CALC, "So the answer")
    \end{lstlisting}
    \vspace{-0.45em}
    \caption{\tool query for arithmetic reasoning.}
    \label{fig:eval_arithmetics:query}
    \end{subfigure}
    \hfill
    \vline
    \hfill
    \begin{subfigure}[t]{0.43\textwidth}
	    \vspace{2.5em}
\begin{adjustbox}{width=1.0\textwidth,center}
    \begin{lstlisting}[language=reactprompt,basicstyle=\ttfamily\footnotesize, escapechar=!]
Q: Noah is a painter. He paints pictures and 
sells them at the park. He charges !\$!60 for 
a large painting and !\$!30 for a small painting. 
Last month he sold eight large paintings and 
four small paintings. If he sold twice as much 
this month, how much is his sales for this month?
A: Let's think step by step.
!\colorbox{my-blue}{He sold 8 large paintings and 4 small}!
!\colorbox{my-blue}{paintings last month.}!
!\colorbox{my-blue}{He sold twice as many this month.}!
!\colorbox{my-blue}{8 large paintings x \$60 = <{}<}\colorbox{my-orange}{8*60=}\colorbox{my-green}{480}\colorbox{my-blue}{ >{}>  480}!
!\colorbox{my-blue}{4 small paintings x \$30 = <{}<\colorbox{my-orange}{4*30=}\colorbox{my-green}{120} >{}>  120}!
!\colorbox{my-blue}{So the answer}! is !\colorbox{my-red}{480}!
    \end{lstlisting}
\end{adjustbox}
	    \vspace{3.8em}
    \caption{Interaction Trace.}
    \label{fig:eval_arithmetics:trace}
    \end{subfigure}
    \vspace{-0.5em}
    \caption{An \tool query implementing on-the-fly evaluation of arithmetic expressions generated by the LM during problem solving steps, addressing a task from GSMK8 \cite{cobbe2021training}. Text in the output, that corresponds to \colorbox{my-blue}{REASON\_OR\_CALC}, \colorbox{my-orange}{EXPR}, \colorbox{my-green}{calculation results} and \colorbox{my-red}{RESULT} is marked in color.}
    \label{fig:eval-calculator}
\vspace{-1em}
\end{figure}

\subsection{Discussion}
Our three case studies show that: i) \tool allows great expressiveness, i.e. several approaches from current state-of-the-art methods can be directly encoded in a straightforward scripting style, requiring much fewer lines of code than corresponding python-based implementations (cf. \cref{tab:eval-loc});
ii) \tool drastically reduces the number of model queries and thereby both efficiency and run time. This is enabled by \tool{}s support for token level validation, which enables us to enforce constraints on-the-fly rather than with chunk-wise decoding and backtracking.
And, iii) that \tool does not impact the accuracy achieved by the model. In fact, in some cases, the enforced constraints even yield slightly improved accuracy. In addition to all this, we have shown that when used in the context of paid, API-gated models, \tool would enable significant monetary savings, given the reduction in billable tokens that we observe.
Lastly, we note that our case studies cannot replace a full user study of \tool, assessing its impact and usability together with real-world prompt engineers. We therefore note that the lack of such a study poses a threat to the validity of our claims with respect to usability.

\section{Related Work} \label{sec:related}

\paragraph{\paradigmlong (\paradigm)} Recent work has proposed a variety of different prompting techniques: chain-of-thought prompting \cite{wei2022chain}, interactive question answering \cite{yao2022react}, aggregation-based schemes like self-consistency \cite{wang2022self}, ThinkSum \cite{ozturkler2022thinksum}, and Iterated Decomposition \cite{reppert_iterated_2023}. 
Recently, a program-aided version of chain-of-thought \citep{gao2022pal,chen2022xxw} with access to a language interpreter was proposed. There, the code output of an LM is fed to an interpreter in order to obtain the answer to e.g. arithmetic tasks by code execution. 
We consider all these works as instances of \paradigm (also discussed under the term of prompt programming \cite{ReynoldsM21, ZhouMHPPCB22}), where the goal is to compose and interact with language models to achieve a specific task. A few select works have identified this trend, and propose novel LM-focused programming systems: PromptChainer \cite{WuJD0MTC22}, \texttt{langchain} \cite{langchain23}, OpenPrompt \cite{ding2021openprompt} and PromptSource \cite{bach2022promptsource} provide integrated development environments or libraries for LM interaction. The latter two even support a simple templating language akin to \tool{} top-level string semantics. However, none of these projects implement constraints or control flow like \tool{} does.
Finally, \citet{dohan2022language} discuss the idea of language model cascades, relating LM querying to probabilistic programming, which opens up interesting avenues for future work, also in the more general context of language model programming and \tool.

\paragraph{Constraining Language Models} The idea of constraining LMs has been applied across a range of fields. \citet{shin2021constrained} constrain a model's output to a more easily-interpretable subset of the English language. More specifically, they handcraft custom next-token prediction programs to implement specific semantic parsing tasks using LMs. \citet{PoesiaP00SMG22} and \citet{scholak2021picard} on the other hand, are concerned with the task of generating source code. In this setting, syntactic and semantic validity is crucial. To realize this, they integrate existing parsers and validation methods. \tool{} on the other hand provides a generic interface to facilitate constrained decoding by providing high-level constructs. Still, our set of operators can easily be extended by the user, allowing for the integration of grammar-based parsers, semantic code validation or other methods.

\section{Conclusion} \label{sec:conclusion}
In this work, we introduce the concept of \paradigmlong, a novel way to interact with (large) language models. We presented \tool{}, a high-level query language, offering a concise and intuitive syntax. \tool{} implements purpose-designed evaluation semantics, which enable efficient query execution. We have substantiated this claim in a series of case studies, where we demonstrate that complex, state-of-the-art prompting techniques can be implemented as intuitive, concise and efficient \tool{} programs that reduce (compute) costs by up to $80\%$. 

\section*{Further Resources} 
With this paper we release our evaluated artifact \cite{lmql_lang_2023_7711823}, our up-to-date codebase at \url{https://github.com/eth-sri/lmql}, an extended updated version at \url{https://arxiv.org/abs/2212.06094} and a project webpage, including live demonstration, at \url{https://lmql.ai}.

\section*{Acknowledgements} 
We thank our colleague Mark Müller for his thoughtful comments and proofreading, and our reviewers and shepard for their service, thoughtful feedback and comments.

This work has received funding from the Swiss State Secretariat for Education, Research and Innovation (SERI) (SERI-funded ERC Consolidator Grant).

\message{^^JLASTBODYPAGE \thepage^^J}

\clearpage
\bibliographystyle{ACM-Reference-Format}
\bibliography{references_}

\message{^^JLASTREFERENCESPAGE \thepage^^J}

\ifbool{includeappendix}{%
	\clearpage
	\appendix
	\section{Implementation}
\label{sec:implementation}

In this section, we discuss a number of technical aspects of our \tool implementation, as can be found at \url{https://github.com/eth-sri/lmql}.

\subsection{Language Runtime}

\paragraph{Parser and Python Compatibility} We implement \tool{} as a superset of python. This also manifests in our implementation, where we rely on the python tokenizer and parser to process LMQL code. Subexpressions in an LMQL query, such as in the \lstinline|where| clause, are parsed as standard python. After some basic program transformations, we emit a python function that interacts with the \tool{} runtime, and allows for interrupted execution by leveraging \lstinline|yield| and \lstinline|async| semantics. This allows us to implement \tool{} as a regular python library, which can be used in any python environment.

\paragraph{Eager Evaluation Semantics} To implement our evaluation semantics, we transform the abstract syntax tree as returned by the python parser into a runtime representation of a computational graph, modelling dependencies among operations explicitly. Users can easily extend \tool{} with custom operators, by implementing a simple class interface with \lstinline|forward|, \lstinline|final| and \lstinline|follow| functions, similar to the integration of custom operators in the popular \lstinline|pytorch| library. Custom operators can easily be registered with the runtime, and the compiler will automatically generate the necessary code to integrate them into the \tool{} computational graph.

\subsection{Model Integration}
\label{sec:model-integration}

\paragraph{Inference API} To enable quick turnaround times during development, \tool{} relies on a client-server-architecture. The server is responsible for inference, loading and managing the model. In our current implementation, it is configured to use a specific HuggingFace Transformers model. Users then interact with the \tool{} client, which is a simple python library. The client parses the user-provided \tool{} code, constructs the computational graph, and also runs the decoding loop. Only the forward pass of the underlying model is outsourced to the server. This naturally aligns with settings in which inference is run on some remote server with capable hardware, while the user interacts with the model via a fast, local client with quick startup times.

\paragraph{Inference as a Service} The underlying client-server architecture of \tool{} also allows for a separation of the \tool{} client and inference as a service. In principle, vendors of API-gated LMs may therefore support \tool{} by providing just the necessary inference API. Alternatively, vendors could accept to-be-executed \tool{} code directly, which would offer customers more control over the decoding process than with current standard APIs. In this context, we consider \tool{} a proposal for the standardization of language model interaction across different vendor-specific APIs. Implementing \tool{} support would allow users to write prompting code once, and run it on any LM platform, without having to change their code. In such a setting, however, we advise for sandboxing of the executed \tool{} queries (like in \emph{serverless computing}), as \tool{} allows for arbitrary code to be executed.

\paragraph{Decoding Loop} \tool{} only requires a small change to existing decoder implementations. For a practical demonstration, see our implementation as published with this paper, in which we adapt the existing HuggingFace Transformers decoding loop to be \tool{}-compatible. In general, \tool{} scripted prompting and output constraining both compile down to token level prediction masks. This is typically already implemented with existing decoders and just needs an additional hook, to call the \tool{} runtime after each produced token. Using this simple interface, \tool{} can be integrated into any decoder implementation, without requiring any changes or retraining of the underlying model.

\subsection{Playground and Visual Debugger}

\begin{figure}
    \centering
    \includegraphics[width=0.8\linewidth]{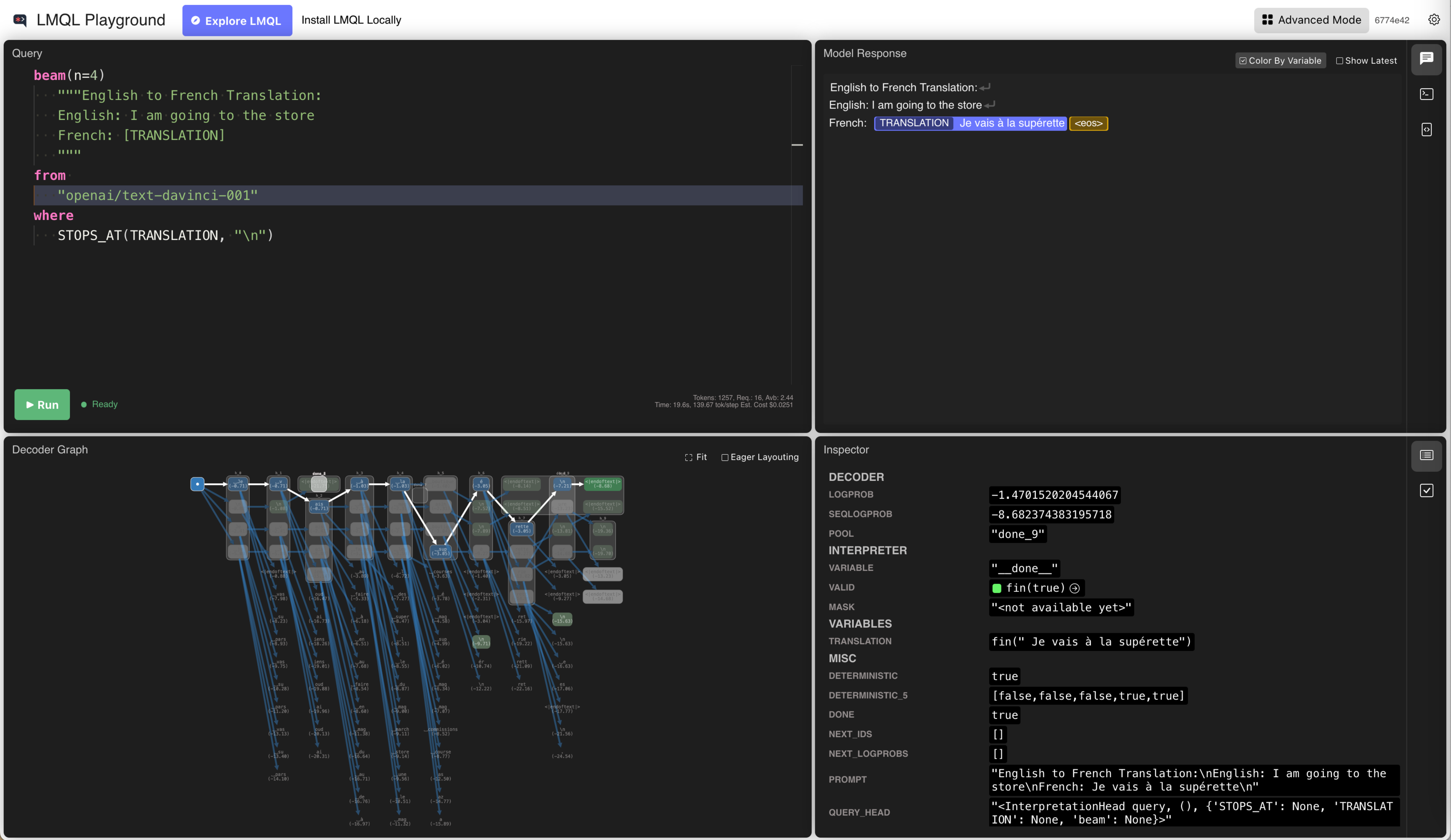}
    \caption{Screenshot of the \tool{} visual debugger in the \tool Playground.}
    \label{fig:debugger}
\end{figure}

Apart from command-line tooling, the \tool{} runtime also includes a web-based playground, helpful in constructing and debugging \tool{} programs. A screenshot of the visual debugger is shown in \cref{fig:debugger}. A hosted version can also be found at \url{https://lmql.ai/playground}.

\paragraph{Editor and Compiler} The visual debugger provides an editor window for constructing \tool{} queries. After a query is executed, users can view the compiler output, i.e. the resulting python code, including the code that constructs the computational graph and executes the prompt.

\paragraph{Decoder Graph} Users can track the different decoding branches of the currently active decoding method in real-time. This includes simple parallel decoding when sampling more than one sequence, but also multi-branch decoding like beam search. The debugger visualizes (sub-)tokens, and at each decoder step, users can inspect the current interaction trace, the value of prompt variables as well as the current state of \lstinline|where| clause validation.

\paragraph{Validation and Masking} Lastly, the computational graph of the \lstinline|where| clause can be visualized and users can track the current value of the expression. In addition to the regular value semantics and partial evaluation, this includes support for both \textsc{Final} and \textsc{Follow} semantics. Different shades of green and red indicate final and non-final \lstinline|True| and \lstinline|False| values, respectively. The \textsc{FollowMap} at each operation can also be inspected, allowing for a detailed analysis of the current state of the computational graph. This can be helpful when developing new \tool{} operators, as it allows for a quick and easy debugging of the underlying semantics.
}{}

\message{^^JLASTPAGE \thepage^^J}

\end{document}